\documentclass[10pt,journal,compsoc]{IEEEtran}

\ifCLASSOPTIONcompsoc
  \usepackage{cite}
\else
  \usepackage{cite}
\fi

\usepackage{cite}

\ifCLASSINFOpdf
\else
\fi
\usepackage{courier}
\usepackage{graphicx}
\usepackage[T1]{fontenc} % optional
\usepackage{amsmath}
\usepackage{amsthm}
\usepackage{mathrsfs}
\usepackage{bm} % optional
\usepackage{amsmath}
\usepackage{graphicx}
\usepackage[ruled, linesnumbered, algo2e]{algorithm2e}
\usepackage{algorithm}
\usepackage{algorithmicx}
\usepackage{algpseudocode}
\usepackage{array}
\usepackage{tabularx}
\usepackage{balance}
\usepackage{makecell}
\usepackage{color}
\usepackage{multirow}
\usepackage{tabularx}
\usepackage{amsfonts}
\usepackage{mathrsfs}
\usepackage{cite}
\usepackage{utfsym}
\usepackage{makecell}

\usepackage{xpatch}

\newtheorem{theorem}{Theorem}[section]

\newtheorem{lemma}[theorem]{Lemma}


\ifCLASSOPTIONcompsoc
\usepackage[caption=false,font=normalsize,labelfon
t=sf,textfont=sf]{subfig}
\else
\usepackage[caption=false,font=footnotesize]{subfi
g}
\fi

\makeatletter
\newcommand{\multiline}[1]{%
  \begin{tabularx}{\dimexpr\linewidth-\ALG@thistlm}[t]{@{}X@{}}
    #1
  \end{tabularx}
}
\makeatother

\begin{document}
\title{Federated Ensemble Model-based Reinforcement Learning in Edge Computing}
\author{Jin~Wang,
        Jia~Hu,
        Jed Mills, 
        Geyong~Min,
        Ming~Xia,
        and~Nektarios~Georgalas % 
\IEEEcompsocitemizethanks{\IEEEcompsocthanksitem Jin Wang, Jia Hu, Jed Mills, and Geyong Min are with the Department
of Computer Science, University of Exeter, United Kingdom.\protect\\
E-mail: \{jw855, j.hu, jm729, g.min\}@exeter.ac.uk 
\IEEEcompsocthanksitem Ming Xia is with Google, California, U.S.A. \protect \\
E-mail: xiaming2006@gmail.com
\IEEEcompsocthanksitem Nektarios Georgalas is with Applied Research Department, British Telecom, United Kingdom. \protect \\
E-mail: nektarios.georgalas@bt.com\protect 
\IEEEcompsocthanksitem Corresponding authors: Jia Hu and Geyong Min. }% <-this % stops an unwanted space
}

%{Shell \MakeLowercase{\textit{et al.}}: Bare Demo of IEEEtran.cls for Computer Society Journals}

\makeatletter
% to justify (both left and right)
\long\def\@IEEEtitleabstractindextextbox#1{\parbox{0.922\textwidth}{#1}} 
\makeatother

\IEEEtitleabstractindextext{%
\begin{abstract}

Federated learning (FL) is a privacy-preserving distributed machine learning paradigm that enables collaborative training among geographically distributed and heterogeneous devices without gathering their data. Extending FL beyond the supervised learning models, federated reinforcement learning (FRL) was proposed to handle sequential decision-making problems in edge computing systems. However, the existing FRL algorithms directly combine model-free RL with FL, thus often leading to high sample complexity and lacking theoretical guarantees. To address the challenges, we propose a novel FRL algorithm that effectively incorporates model-based RL and ensemble knowledge distillation into FL for the first time. Specifically, we utilise FL and knowledge distillation to create an ensemble of dynamics models for clients, and then train the policy by solely using the ensemble model without interacting with the environment. Furthermore, we theoretically prove that the monotonic improvement of the proposed algorithm is guaranteed. The extensive experimental results demonstrate that our algorithm obtains much higher sample efficiency compared to classic model-free FRL algorithms in the challenging continuous control benchmark environments under edge computing settings. The results also highlight the significant impact of heterogeneous client data and local model update steps on the performance of FRL, validating the insights obtained from our theoretical analysis.
\end{abstract}

\begin{IEEEkeywords}
Edge computing, distributed machine learning, federated learning, deep reinforcement learning
%Computer Society, IEEE, IEEEtran, journal, \LaTeX, paper, template.
\end{IEEEkeywords}
}

\maketitle
\IEEEdisplaynontitleabstractindextext

% \fi
\IEEEpeerreviewmaketitle

\IEEEraisesectionheading{\section{Introduction}\label{sec:introduction}}
% \IEEEraisesectionheading puts \section within a raised box. \IEEEPARstart{F}
  
\IEEEPARstart{T}he advancements in deep learning (DL) \cite{lecun2015deep} algorithms and high-performance computing technologies are fundamental to the tremendous successes of artificial intelligence (AI) in many aspects of our societies, including transportation, healthcare, education, etc. The emerging AI-empowered applications such as smart manufacturing, autonomous driving, and smart healthcare generate large volumes of data on the user side. To enable real-time data processing for these emerging applications, edge computing was proposed to shift computation and storage resources from the remote Cloud to the network edge in the proximity of end-users. Traditional centralized AI approaches need to collect data from end-users and save it centrally at edge servers to effectively train DL models for various applications. However, users are often unwilling to share their sensitive data with others due to the growing concern on data privacy, thus rendering these centralized approaches impractical in many cases. 

To address the aforementioned issue, federated learning (FL) was proposed to collaboratively train DL models in a distributed fashion without sensitive data leaving the user devices. In FL, models are trained locally at clients (i.e., user devices) and only the model parameters are uploaded by clients to the server. The existing FL works \cite{adaptiveFedOpt,convFedAvgNonIID, AdvancesinFL,liu2020accelerating,duan2020self} predominantly consider training supervised learning models (e.g., Convolutional Neural Networks and Long Short-Term Memory) for solving perception problems such as image classification and linguistic prediction. 

More recently, federated reinforcement learning (FRL) was proposed to extend FL to train reinforcement learning (RL) models for solving sequential decision-making problems in edge computing, such as resource allocation \cite{cui2021secure,yu2020deep}, content caching \cite{AWFDRL}, and user access control \cite{cao2021user}. Those studies directly combine model-free RL (learning without using a system dynamics model) with FL. Specifically, they train policies locally for all collaborating devices, using the model-free RL objective, and average the policy parameters on the server to generate a global policy for the next round of local training. However, traditional model-free RL algorithms generally have high sample complexity whilst obtaining samples is costly in many real-world edge computing scenarios such as smart factories and intelligent transport. For example, when applying RL methods to solve the task offloading problem in edge computing \cite{yu2020deep}, the immediate reward for an agent can only be obtained once the offloaded task is executed. Obtaining an effective offloading policy via model-free RL may require numerous trial-and-error steps where the agent interacts with the targeted edge computing system, resulting in huge costs. Besides, the theoretical properties (such as monotonic improvement) of these model-free FRL algorithms were not well understood. These issues hinder the practical use of model-free FRL in real-world edge computing scenarios. 

Compared to model-free methods, model-based RL \cite{janner2019trust,kurutach2018model,kaiser2020model} is much more sample efficient. Model-based RL learns an estimated dynamics model and then derives an optimal policy based on the learned model. Since the dynamics model is trained by using supervised learning, it can be naturally adapted to the current federated supervised learning setting where many state-of-the-art FL algorithms are available. In addition, when applying model-based FRL in edge computing, training of the RL policy can be offline (as the training process is based on interactions with the learned dynamics model), saving the huge costs of interacting directly with the edge computing system. 

Despite its promising benefits, there are several major challenges for effectively integrating model-based RL into FL. First, model bias (caused by overfitting in regions where insufficient data is available to train the model) is a key factor that affects model-based RL methods \cite{kurutach2018model}. Handling RL model bias in the federated setting is even more challenging due to the highly heterogeneous client data. Second, a rigorous theoretical analysis of federated RL is lacking. Especially, monotonic improvement of RL algorithms has not been proven to hold in the federated setting. Third, it is unclear how non-independent and identically distributed (non-IID) client data will affect the performance of federated model-based RL. 

In this paper, we extend model-based RL to the revolutionary FL paradigm, proposing a novel federated ensemble model-based reinforcement learning (FEMRL) algorithm. In FEMRL, the dynamics model is trained by FL, and then the RL policy is trained by solely using the dynamics model without interacting with the environment. To address the problem of model bias, we create an ensemble of dynamics models uploaded by clients. In addition, an ensemble distillation method is used to enhance the performance of model aggregation during FL. We summarise the key contributions of our work as follows:

\begin{itemize}
    \item To the best of our knowledge, this is the first of its kind that effectively extends model-based RL to the popular FL setting. In particular, we integrate FL and knowledge distillation techniques to create an ensemble of dynamics models for clients and then train the policy by solely using the ensemble without relying on the costly process of sampling data from the environment.

    \item We provide a rigorous theoretical analysis to prove that the monotonic improvement of FEMRL is guaranteed. The discrepancy bound of the return from the environment and the learned dynamics identifies and highlights the impacts of non-IID client data on the policy improvement for federated RL.

    \item We perform extensive experiments using four challenging continuous control environments \cite{Mujoco} under edge computing settings. The results demonstrate the superior sampling efficiency (hence lower computation and communication cost) of FEMRL compared to classic model-free FRL algorithms. The results also highlight the significant impacts of non-IID client data and local model update steps on the rate of reward improvement for federated RL, validating the insights obtained from our theoretical analysis.
\end{itemize}

The rest of the paper is organised as follows. Section \ref{related work} introduces the related work including federated learning in edge computing systems and model-based reinforcement learning. We next overview some necessary background knowledge related to FL and RL in Section \ref{Preliminaries}. Section \ref{methodology} presents details of the proposed FEMRL including the algorithm design and theoretical analysis. We then evaluate FEMRL with four standard RL environments and give the discussion about the experimental results in section \ref{sec::experiment-sec}. Finally, we summarise the paper in section \ref{conclusion}.

\begin{table*}[ht]
  \renewcommand\arraystretch{1.1}
  \renewcommand{\tabcolsep}{1.5mm}
  \caption{{\color{black}A summary of differences between the related work and our work.}}
  \label{summary_related_work}
  \centering
  \begin{tabular}{ p{5cm} | c | c | c | c }
    \hline
    \makecell[c]{\textbf{Related research topics}} & \textbf{References} & \textbf{Support federated training} & \textbf{Support decision-making} & \textbf{Sample efficiency} \\ 
    \hline
    \makecell[c]{Federated reinforcement learning} & \cite{FedRLFastPersonalisation,FedDeepRL,cui2021secure,AWFDRL,cao2021user} &  \checkmark &  \checkmark & low \\
    \hline
    \makecell[c]{Model-based reinforcement learning} & \cite{kurutach2018model,chua2018deep,zhang2019asynchronous,luo2018algorithmic,janner2019trust,sun2018dual,kidambi2020morel} &  $\times$ & \checkmark & high \\
    \hline
    \makecell[c]{Federated ensemble distillation} & \cite{chen2021fedbe,lin2020ensemble} &  \checkmark& $\times$ & high\\
    \hline
    \makecell[c]{Federated ensemble model-based \\ reinforcement learning} & our work&  \checkmark & \checkmark & high\\
    \hline
  \end{tabular}
\end{table*}

\section{Related Work}
\label{related work}
{\color{black} The related work focuses on extending RL algorithms to FL settings in edge computing systems, namely federated reinforcement learning. However, directly combining model-free RL with FL has low sample complexity. This work aims to improve the sample efficiency by adapting model-based RL to FL settings and further improve the training stability by utilizing federated ensemble distillation. In Table \ref{summary_related_work}, we summarize the related research topics of Federated Reinforcement Learning, Model-based Reinforcement Learning, and Federated Ensemble Distillation and present the detailed review in the following paragraphs. }

\textbf{Federated Reinforcement Learning:} Several previous studies have investigated training RL policies in the FL setting. Nadiger \textit{et al.} \cite{FedRLFastPersonalisation} proposed a system for training virtual Pong players (controlled via a Deep Q-network) in the FL setting to match the skill levels of (simulated) players. The authors in \cite{FedDeepRL} designed the FedRL system for training a policy, where individual FL clients do not have access to the full state-space of the RL task. Some researchers focus on domain-specific federated reinforcement learning in edge computing system. In \cite{cui2021secure}, the authors combined federated reinforcement learning and blockchain to solve resource allocation problem in edge computing system, providing reliable and secure training process. Wang \textit{et al.} \cite{AWFDRL} proposed an attention-weighted federated deep reinforcement learning model to solve the  heterogeneous collaborative edge caching problem by jointly optimising the node selection and cache replacement in device-to-device assisted mobile networks. In \cite{cao2021user}, the authors proposed an intelligent user access control scheme based on FRL in radio access networks to optimise the overall throughput and avoid frequent handovers. Whilef these works contribute to the development of model-free RL in the FL setting, they suffer from high sample complexity and lack theoretical guarantees. 
\\

\noindent \textbf{Model-based Reinforcement Learning:} RL algorithms are generally built on Markov Decision Processes (MDP) and can be divided into two categories: model-free RL algorithms, which directly train a value function or policy by trial-and-error in the environment; and model-based RL algorithms that explicitly learn a dynamics model based on the sampled data and derive a policy from the model. Model-based RL has been demonstrated to have significantly higher sample efficiency than model-free RL, and has been successfully applied to robotics \cite{zhang2019asynchronous}, video games \cite{kaiser2020model}, etc., using a variety of dynamics models including Gaussian processes \cite{deisenroth2011pilco}, linear models \cite{levine2014learning,tassa2012synthesis}, mixtures of Gaussians \cite{khansari2011learning}, and Deep Neural Networks (DNNs) \cite{depeweg2017learning,draeger1995model,nagabandi2018neural}. One key challenge for model-based RL is how to handle uncertainty of the dynamics model \cite{kurutach2018model,chua2018deep}. To address this challenge, ensembles of DNNs \cite{kurutach2018model,chua2018deep,zhang2019asynchronous} have been used to handle model uncertainty given data collected from the environment. In our FEMRL algorithm, we approximate the model dynamics using DNNs and create an ensemble using the models uploaded by FL clients. From the theoretical perspective, previous works \cite{luo2018algorithmic,janner2019trust,sun2018dual,kidambi2020morel} have provided general frameworks for analysing model-based RL, which include monotonic improvement guarantees. We extend the analyses of these works to our FEMRL algorithm, proving the monotonic improvement of FEMRL, which also demonstrates the influence of non-IID client data on the policy improvement.
\\

\noindent \textbf{Federated Ensemble Distillation:} FL aims to train a global model by sharing users' locally-trained models, rather than their private data. A crucial step in FL is how to aggregate local models into a global model. The seminal FedAvg algorithm \cite{fedavg} averages local models after each communication round to produce a new global model. However, directly averaging model parameters may not be the most effective method of creating the global model, due to non-IID client data, which is a significant challenge in FL and can come in many forms \cite{quagmire}. Some recent works focus on using ensemble distillation techniques to create more robust global models. \cite{chen2021fedbe} proposed a novel aggregation approach using Bayesian model ensembles and knowledge distillation. \cite{lin2020ensemble} proposed a similar algorithm for distillation on the server, using the average logits of the client models on an unlabelled dataset as the distillation target. Inspired by the above methods, we aggregate the client models into a single global model using knowledge distillation. Moreover, in our method, we sample fictional experience (as opposed to real experience) from the ensemble of models for knowledge distillation, further helps reduce the privacy risks of FEMRL.

\section{Preliminaries}
\label{Preliminaries}
In this section, we provide some necessary background about the formulations of FL and RL problems. 

\subsection{Federated Learning}
In FL, clients collaboratively train a model without exchanging their training data in any way. The FL objective is to find the minimiser $\bm{w}$ of the average client loss function $f$:
\begin{equation}
  \label{fl_objective}
  \underset{w \in \mathbb{R}^d}{\min} f(\bm{w}) = \frac{1}{K} \sum_{k = 1}^{K} p_k f_k(\bm{w}),
\end{equation}
where $K$ is the total number of clients, $p_k$ and $f_k$ are the fraction of total samples $( \sum_{k} p_k = 1 )$ and average loss over samples on client $k$, respectively. Therefore, FL aims to compute the minimiser of the average loss over all samples on all participating clients (i.e., the same objective as would be achieved by centralised training on pooled data). However, in real-world FL data is non-IID across clients, as the behaviour of each client influences how its local samples are generated. Non-IID client data has been extensively shown to hinder the convergence of the FL model, and is one of the key challenges to FL. In our FEMRL algorithm, we use FL to train the dynamics model of the MDP.

\subsection{Reinforcement Learning} 

A sequential decision-making problem solved by RL is generally modelled as an MDP, which is given by the six-tuple $\mathcal{M} := (\mathcal{S}, \mathcal{A}, T, R, \rho_0, \gamma)$. Here, $\mathcal{S}$ and $\mathcal{A}$ are the state and action spaces, respectively. $T(s'|s,a)$ represents the dynamics that specifies the conditional distribution of the next state $s'$ given the current state $s$ and action $a$. $R(s, a)$ is the reward function, $\rho_0$ represents the initial state distribution, and $\gamma \in (0, 1)$ denotes the discount-factor. Denote $\pi(\cdot|s)$ as the policy that specifies the conditional distribution over action space given a state $s$. The goal of RL algorithms is to find the optimal policy that maximises the expected discounted return defined by $\mathbb{E}_{\pi, T, \rho_0}\left [ \sum_{t=0}^{\infty} \gamma^t R(S_t, A_t) \right ]$. Define the value function following policy $\pi$ with MDP $\mathcal{M} := (\mathcal{S}, \mathcal{A}, T, R, \rho_0, \gamma)$ as:
\begin{equation}
  \label{value_function} 
  V_{\pi}^{\mathcal{M}}(s) = \mathop{\mathbb{E}}\limits_{ \tiny \begin{array}{c} S_{t+1} \sim T(\cdot | S_t, A_t) \\ A_t \sim \pi(\cdot | S_t) \end{array} } \left [\sum_{t=0}^{\infty} \gamma^t R(S_{t}, A_{t}) \bigg\rvert S_0 = s \right ].
\end{equation}
Thus $V_{\pi}^{\mathcal{M}} := V_{\pi}^{\mathcal{M}}(s_0)$ is the total return given policy $\pi$, where $s_0 \sim \rho_0$ is the initial state.

\section{Federated Ensemble Model-based Reinforcement Learning (FEMRL)}
\label{methodology}
\begin{figure*}[t]
  \centering
  \includegraphics[width=6.0in]{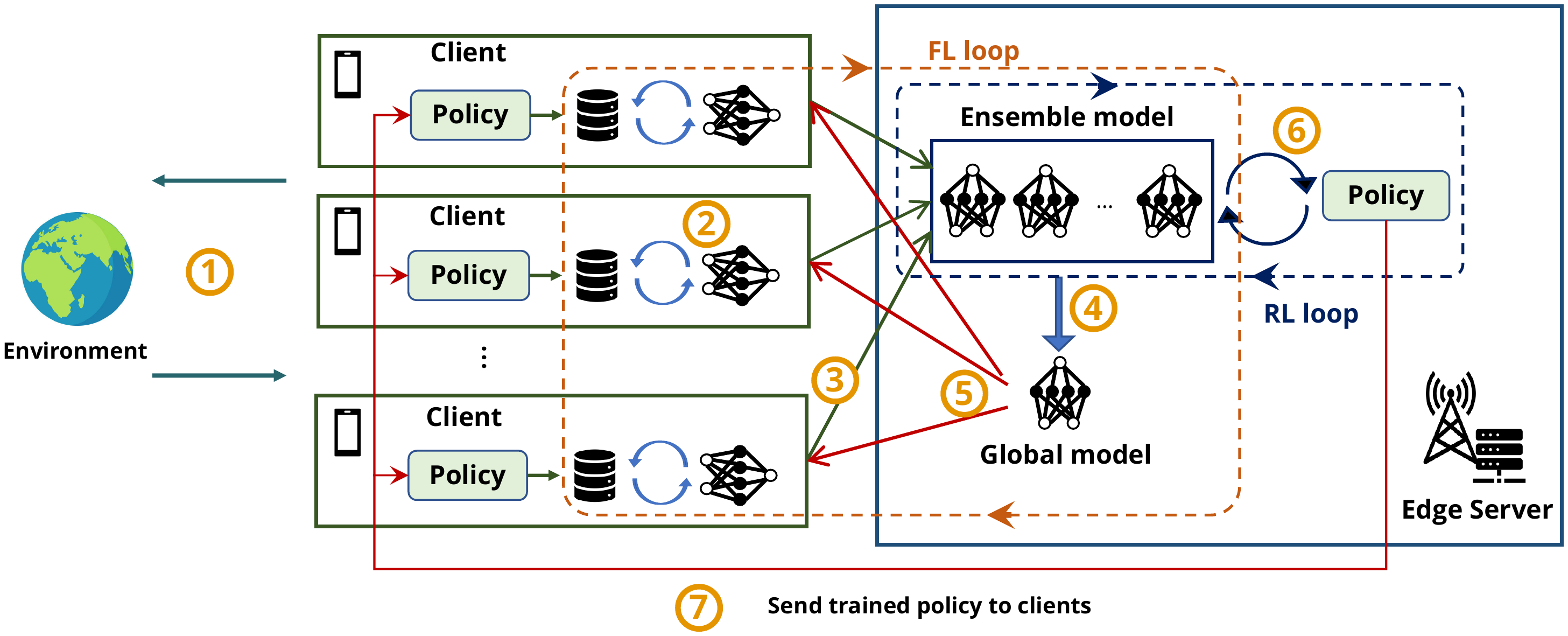}
  \centering\caption{Overview of the FEMRL algorithm. \textbf{Step 1:} each client samples data from the environment based on the local sample policy and stores the data locally. \textbf{Step 2:} local dynamics models are trained based on the sampled data. \textbf{Step 3:} the parameters of the local dynamics models are sent to the server. \textbf{Step 4:} an ensemble of dynamics models are created on the server using the uploaded client models, and a single global model is then created via knowledge distillation. \textbf{Step 5:} the parameters of the global model are sent to clients. Then, starting step 2 again for $T_{\rm c}$ rounds of FL loops. \textbf{Step 6:} after rounds of FL training, the server then trains the policy using a policy-gradient algorithm (e.g., TRPO) and the ensemble of dynamics models. \textbf{Step 7:} the parameters of the new policy are sent to clients for the next round of sampling (i.e., Step 1). }
  \label{fedmbdrl_architecture}
\end{figure*}
In this section, we describe the proposed FEMRL algorithm in detail, and then provide a theoretical analysis guaranteeing monotonic improvement of the policy produced by FEMRL. 

\subsection{Algorithm Design}
Our algorithm intends to train a model-based RL policy in an edge computing environment involving multiple client devices, and a corresponding edge server. {\color{black} In our setting, all participating clients share the same environment with different state transitions}. There are many real-world applications corresponding to this setting, including unmanned aerial vehicles \cite{mowla2020afrl}, edge caching \cite{wang2020federated}, user access control \cite{cao2021user}, and resource management \cite{yu2020deep,cui2021secure} in edge computing systems. Fig. \ref{fedmbdrl_architecture} illustrates the operation of FEMRL, which consists of two major sub-components: 
FL loop for the training of dynamics model, and RL loop for policy training. 

Formally, define the MDP with learned dynamics $\widehat{T}(s'|s, a; \bm{w})$ as $\widehat{\mathcal{M}} := (\mathcal{S}, \mathcal{A}, \widehat{T}, R, \rho_0, \gamma)$, where $\bm{w}$ are the parameters of the learned model. Define $\widehat{T}(s, a; \bm{w})$ as the function that produces the unique value of $s'$. The goal of the FL loop is to learn the optimal $\bm{w}$ such that the discrepancy between the learned dynamics and real dynamics is minimal. This minimisation is a typical supervised learning process, which can be solved through maximum likelihood estimation or other techniques from generative and dynamics modelling. In this paper, we apply a multi-step prediction loss that is similar to \cite{luo2018algorithmic} for model learning, and use a predefined reward function, as in the works \cite{kurutach2018model,chua2018deep,luo2018algorithmic}. Concretely,s for a state $s_t$ and action sequence $a_{t:t+h}$, the $h$-step prediction $\hat{s}_{t+h}$ as $\hat{s}_t = s_t$, and for $h \ge 0$, $\hat{s}_{t+h+1} = \widehat{T}(\hat{s}_{t+h}, a_{t+h}; \bm{w})$, the {\color{black}$H$-step} loss is defined as:
\begin{equation}
\label{loss_func}
    f(\bm{w}) = \frac{1}{H} \sum_{i=1}^H \left \| (\hat{s}_{t+i} - \hat{s}_{t+i-1}) - (s_{t+i} - s_{t+i-1}) \right \|_2.
\end{equation}

{\color{black}The FL loop involves $T_{\rm c}$ rounds of communication between client devices and the edge server. Within each round of federated training, each client parallelly conducts the local update procedure as shown in Algorithm \ref{client-side}. The client first samples trajectories from the environment using the current policy $\pi_{D} \leftarrow \pi_{\theta}$, where $\pi_{\theta}$ is the updated policy received from the server. The client then collects all the sampled trajectories into the local replay buffer, $D_k$. Note that the distribution of sampling trajectories is determined by the values of the policy parameters $\theta$ and the dynamics of the environment $T(s'|s,a)$ as:  
\begin{equation}
	\label{dist_sample_trajectories}
	P(s_0,a_0,s_1,..., s_{n},a_{n},s_{n+1}) = \rho_0 \prod_{t=0}^{n} \pi_{\theta}(a_t|s_t)T(s_{t+1}|s_t, a_t).
\end{equation}
Next, the client conducts $E$ local update steps to train the local dynamics model with mini-batch gradient descent. The returned local dynamics model is then uploaded to the server for further processing.} All uploaded models are then aggregated into a single global model on the server-side. Instead of simply averaging the local models as in FedAvg \cite{fedavg}, we create an ensemble model $\{ \bm{w}_{k} \}_{k=1}^{m}$ based on the uploaded local models, where $\bm{w}_k$ is the local model updated by the $k$\textsuperscript{th} client. This ensemble serves two purposes: 1) creating a single global dynamics model that benefits from knowledge distillation; 2) generating fictitious data for policy training. Using the model ensemble, therefore, benefits both the FL and policy training processes by producing a robust aggregate model and alleviating the model bias problem in policy training. {\color{black} In our proposed FEMRL, the policy is trained through interacting with the learned dynamics model rather than the actual environment. Therefore, the model error has a significant impact on the learned policy. To reduce the impact of the model error, the ensemble method provides an effective regularization for policy training: by using the ensemble dynamics model, the policy is able to perform well over many possible alternative futures, making the learned policy more robust.}

{\color{black} The ensemble knowledge distillation method involves a typical student-teacher learning scheme. Denote the sampled fictitious data as $\mathcal{D} = \{s_0, a_0, ...., s_n, a_n\}$, $s_0 \sim \rho_0$, $a_t \sim \pi(a_t|s_t)$, $s_{t+1} = \widehat{T}(s_t, a_t; \{ \bm{w}_{k} \}_{k=1}^{m})$. The student model (i.e., the single global dynamics model) is trained with Adam \cite{adam} following the loss function:
\begin{equation}
    \label{distillation_loss}
    L(\overline{\bm{w}}) = \left \| \frac{1}{m}\sum_{k=0}^mT(s_t, a_t; \bm{w}_k) - T(s_t, a_t; \overline{\bm{w}}) \right \|_2,
\end{equation}
where $T(s_t, a_t; \bm{w}_k)$ is the learned local dynamics of client $k$ and $T(s_t, a_t; \overline{\bm{w}})$ is the global dynamics represented by the student model.
}

After $T_{\rm c}$ rounds of federated training, we then use a policy-gradient algorithm (Trust Region Policy Optimization (TRPO) \cite{schulman2015trust}) to train the policy by interacting with the ensemble of models. Next, the parameters of the updated policy are sent to all participating clients, which will then start the next round of sampling procedure using the updated policy. We adopt asynchronous model aggregation where the server does not wait for all clients to finish sending their updated local models. {\color{black} At each training round, only a fraction, $\alpha$ (i.e., policy synchronisation rate), of clients update their policy using the newest global policy. This design is practical since clients can be unreliable edge devices that may not always be able to reach the server (e.g., a smartphone loses its network connection) in the FL scenario. For $\alpha < 1$, clients' data distributions become non-IID, as some clients will be performing local updates on the environment model using a ‘stale’ (unsynchronised) policy. } We present the detailed server-side algorithm of FEMRL in Algorithm \ref{server-side procedure}. Specifically, we conduct training with $n_{\rm outer}$ epochs. Each epoch involves $n_{\rm inner}$ rounds of inner loops. Within each inner loop, we alternatively conduct $T_{\rm c}$ rounds of FL loops and $G$ rounds of RL loops.

\begin{algorithm}[t]
	{\color{black}
%\SetAlgoNoLine
  \caption{FEMRL running on $K$ clients (indexed by $k$) for $E$ epochs,  each consisting of $T_{\rm c}$ rounds of federated communication and $G$ steps of policy update.  }
  \label{server-side procedure}
  \SetKwProg{modeltraining}{Procedure}{}{}
  \SetKwProg{policytraining}{Procedure}{}{}
  \SetKwProg{generatefictitiousData}{Procedure}{}{}

  \modeltraining{\rm FEMRL}{
    \For{$n_{\rm outer}$ \text{\rm epochs} }{
        \For{$n_{\rm inner}$ \text{\rm iterations} } {
            $\{\bm{w}^{(k)}\}_{k=1}^K \leftarrow$ \text{FedEnLearning}($T_{\rm c}$) \\
            \For{$G$ \text{\rm iterations}} {
              Generate fictitious samples $\mathcal{D} \leftarrow \text{\textit{GenerateFictitiousData}}(\{\bm{w}_k\}_{k=1}^K, \pi_\theta)$. \\
              Update policy $\pi_\theta$ using TRPO and $\mathcal{D}$
              }
        }

         Send the updated policy $\pi_{\theta}$ to clients with synchronisation rate $\alpha$.
      }
  }
\ \ \ \\
  \modeltraining{\rm FedEnLearning $(T_{\rm c})$}{
  	Initialise parameters of the student model $\overline{\bm{w}}$ \\
    \For{$T_{\rm c}$ \text{\rm iterations}}{
        \For{ \text{\rm each client} $k \in K$ {\bf{in parallel} } } {
            $\triangleright$ \textit{ LocalUpdate is detailed in Algorithm \ref{client-side}} \\
            $\triangleright$ \textit{ At each local update round, the student model $\overline{\bm{w}}$ works as initial model of all participated clients.}\\
            $\bm{w}_k \gets$ $\text{\textit{LocalUpdate}}(k, \overline{\bm{w}}, E)$ 
        }

        Create ensemble of models $\{\bm{w}_k\}_{k=1}^K$ \\
        \For{$N$ \text{\rm iterations}} {
            Generate fictitious samples $\mathcal{D} \leftarrow \text{\textit{GenerateFictitiousData}}(\{\bm{w}_k\}_{k=1}^K, \pi_\theta)$. \\
            $\triangleright$ \textit{ The updated student model $\overline{\bm{w}}$ is then used by LocalUpdate procedure for next-round of local training.}\\
            Update the student model  $\overline{\bm{w}}$ using loss function from Eq. (\ref{distillation_loss}) on $\mathcal{D}$.
        }
        } 
        \textbf{return} $\{\bm{w}_k\}_{k=1}^K$
  }
  \ \ \ \\
  \generatefictitiousData{\rm GenerateFictitiousData $(\{\bm{w}_k\}_{k=1}^K, \pi_\theta)$}{
    Sample initial state $s_0$ from the initial state distribution $s_0 \sim \rho_0$\\
    \For{$t \gets 0$ \rm{\textbf{to}} $N$} {
        Sample $a_t \sim \pi_\theta(a_t|s_t)$ from policy $\pi_\theta$ \\
        Randomly sample a dynmics model $\bm{w}^{(k)}$ from the ensemlbe of models $\{\bm{w}_k \}_{k=1}^K$ \\
        Using the dynamics model $\bm{w}_k$ to predict the next state $s_{t+1} \sim \widehat{T}(s_{t+1}|s_t, a_t; \bm{w}_k)$ \\
        Get reward $r_t$ by the reward function $r_t = R(s_t, a_t)$ \\
        Add the transition to fictitious dataset $\mathcal{D} \bigcup \{s_t, a_t, r_t, s_{t+1}\}$
    }
     \textbf{return} $\mathcal{D}$
  }
}
\end{algorithm}

\begin{algorithm}[t]
  %\SetAlgoNoLine
  { \color{black} 
  \caption{Procedures of client side}
  \label{client-side}
  
  \SetKwProg{localupdate}{Procedure}{}{ }
  \localupdate{$\text{\rm LocalUpdate} (k, \omega_{k}^{0}, E)$}{
    Sample initial state $s_0$ from the initial state distribution $s_0 \sim \rho_0$. \\
    \For{$t \gets 0$ \textbf{to} $N$} {
    	Sample $a_t \sim \pi_{D}(a_t|s_t)$ with current policy $\pi_D \leftarrow \pi_\theta$. \\
    	Apply $a_t$ to the environment and get the next state $s_{t+1}$ and reward $r_t$. \\
    	Store the transition to the local replay buffer $D_k \gets D_k \cup (s_t, a_t, r_t, s_{t+1})$.
    	}
    \For{$i \gets 1$ \textbf{to} $E$} {
      Random sample a batch of training data $\xi_i$ from $D_k$ \\
      Conduct mini-batch gradient descent: $ \omega_{k}^{i} \gets \omega_{k}^{i-1} - \eta_i \nabla f_k(\omega_{k}^{i-1}; \xi_i)$
    }
    
    \KwRet $\omega_{k}^{E}$
}
  }
\end{algorithm}

{\color{black} It is noteworthy that the model-free RL methods can also be integrated into the framework as follows. First, the server receives locally trained policy networks from clients and creates an ensemble of policy networks. Next, a single global policy network is created via knowledge distillation. Finally, the parameters of the global policy network are sent to clients, starting next-round local training. In the following sections, we provide a theoretical guarantee of monotonic policy improvement for FEMRL, before performing a thorough empirical evaluation of the algorithm. }

\subsection{Theoretical analysis}
Proving monotonic improvement guarantee is an important aspect of RL algorithms. In this section, we provide the conditions under which FEMRL is guaranteed to provide monotonic improvement for $\pi$. To prove monotonic improvement of a model-based RL algorithm, we wish to find a lower bound of $V_{\pi}^{\mathcal{M}}$:
\begin{equation}
    \label{target_func}
    V_{\pi}^{\mathcal{M}} \ge V_{\pi}^{\widehat{\mathcal{M}}} - B
\end{equation}
where $B$ is the bounded value. 

Since the model is trained with supervised learning, the distance between the true model and the learned model can be quantified by standard Probably Approximately Correct (PAC) generalization error \cite{shalev2014understanding}. PAC bounds the difference in generalisation and empirical error by a constant with high probability. In FEMRL, this generalisation error can be defined as the distance between the learned dynamics and the environment dynamics. The recent literature provides two main ways to measure this distance, each with different assumptions. One assumes that the dynamics model is a complex probability distribution, and measures the distance using Total Variation Distance (TVD)  \cite{janner2019trust}. The other assumes deterministic dynamics and directly uses 1-Wasserstein distance \cite{luo2018algorithmic}. In addition, \cite{yu2020MOPO} uses a general measurement, Integral Probability Metric, where TVD and 1-Wasserstein distance are two special cases. Since TVD requires weaker assumptions and is typically more practical than 1-Wasserstein distance, we use TVD in our analysis. Overall, we make the following assumptions:
\\ 

\noindent \textbf{Assumption 1.} The generalisation error is measured by the TVD, defined as $\epsilon_m := D_{\rm TV}(\widehat{T}(\cdot | s,a)|T(\cdot | s,a)) = \frac{1}{2} \sum_{s'} \left | \widehat{T}(s' | s,a) - T(s' | s,a) \right |$
\\

\noindent \textbf{Assumption 2.} The dependency of two policies $\pi$ and $\pi_{D}$ is measured by the TVD $\epsilon_\pi = D_{\rm TV}(\pi(a|s)|\pi_{\rm D}(a|s))$, and is bounded by a constant $\delta_{\pi}$, where $D_{\rm TV}(\pi(a|s)|\pi_{ D}(a|s)) \leq \delta_{\pi}$.
\\

\noindent \textbf{Assumption 3.} The reward function of the MDP is bounded: $\forall s \in \mathcal{S}, \forall a \in \mathcal{A}, R(s,a) \leq r_{\rm max}$.
\\

\noindent \textbf{Assumption 4.} The loss function of the FL dynamics model is convex and bounded by $L$, $|f(\bm{w})| \leq L$, $\forall \bm{w}$.
\\

\noindent Based on previous works \cite{janner2019trust,luo2018algorithmic,yu2020MOPO}, we have the following Lemma to build the lower bound of the discrepancy of the total returns from the true model and the learned model in conventional model-based RL:

\begin{lemma}
    \label{lemma1}
    Denote $\epsilon_{m}$ as the generalization error of the dynamics model and $\epsilon_{m}^{\rm max}$ as the maximal value of $\epsilon_{m}$. Denote $\epsilon_{\pi}$ as the discrepancy between target policy  $\pi$ and sample policy $\pi_{D}$. For any policy $\pi$, the  return of the environment $V_{\pi}^{\mathcal{M}}$ and the return of the learned dynamics $V_{\pi}^{\widehat{\mathcal{M}}}$ are bounded as:
    \begin{equation}
        V_{\pi}^{\mathcal{M}} \geq V_{\pi}^{\widehat{\mathcal{M}}} - \underbrace{\left [ \frac{2 \gamma r_{\rm max }}{1-\gamma} \epsilon_{m} + \frac{4 \gamma^2 r_{\rm max }}{(1-\gamma)^3} \epsilon_{\pi} \epsilon_{m}^{\rm max} \right ]}_{B}.
    \end{equation}
\end{lemma} 

\begin{proof}
    See Appendix \ref{appendix:lemma1s}. 
\end{proof}

\noindent Lemma \ref{lemma1} gives a theoretical guarantee for the monotonic improvement of the model-based RL algorithm. As long as we improve the returns under the learned model by more than $B$, we can guarantee improvement under the environment \cite{janner2019trust}. The bound $B$ is proportional to the generalization error of the dynamics model, $\epsilon_m$, and the discrepancy between the sample policy and target policy, $\epsilon_\pi$. However, Lemma \ref{lemma1} holds only if the generalization error $\epsilon_{m}$ is bounded.  Conventional model-based RL methods use normal centralised supervised learning to train the dynamics model, however, in FEMRL we use FL to train the dynamics model through an ensemble of models created from the clients' local models to approximate the learned model, $\widehat{T}(s'|s, a; \{\bm{w}^{(k)}\}_{k=1}^K)$. Therefore, it is necessary to investigate if $\epsilon_{m}$ is bounded in the FL setting and what factors influence $\epsilon_{m}$ in FEMRL. 

We now derive a bound on the generalisation error of the ensemble of client models.
 
\begin{theorem}
\label{theorem1}
Denote the global data distribution as $D$. Let $D_k$ be the local data distribution of client $k$. {\color{black}Let $\pi_{D}^k$ be the sample policy for client $k$. Let $\overline{\pi}_{D}$ be the virtual global sample policy.} Therefore, we have $D = \mathbb{P}_{s,a,s'} = \sum_{s,a} T(s'|s,a) \overline{\pi}_{D}(a|s)$ and $D_k = \mathbb{P}_{s,a,s'} = \sum_{s,a} T(s'|s,a) \pi_{D}^k(a|s)$. Denote $S_k \sim D_{k}^{m}$ as local empirical distribution for client $k$. Let $\hat{S}$ be the global empirical distribution, each local empirical distribution has equal contribution to the global distribution, thus $\hat{S} = \frac{1}{K}\sum_{k=1}^{K} S_{k}$. Let $\mathcal{H}$ be a hypothesis class with limited Vapnik–Chervonenkis (VC) dimension, $VCdim(\mathcal{H}) \leq d < \infty$. The hypothesis $h\in \mathcal{H}$ learned on $S_k$ and $\hat{S}_k$ is denoted by $h_{S_k}$ and $\hat{h}_{S_k}$, respectively. Then, the generalisation error of the ensemble model is bounded with probability at least $1-\delta$:
\begin{equation}
\label{theory_1}
\begin{aligned}
 \epsilon_{m} &:=  \epsilon_{D}(\frac{1}{K} \sum_{k}{h_{S_k}} ) \\
              &\leq \epsilon_{\hat{S}_k}(h_{\hat{S}_k}) + C \sqrt{\frac{d + log(1/\delta)}{m}} + \frac{L}{K} \Gamma,
\end{aligned}
\end{equation}
where $C$ and $L$ are constants, $m$ is the number of training samples per local data distribution, and {\color{black}$\Gamma = \sum_{k=1}^K  D_{\rm TV}(\overline{\pi}_{D} ||\pi_{D}^{k})$ which is affected by the sample policies.}

\end{theorem}

\begin{proof}
    See Appendix \ref{appendix:gen}. 
\end{proof}

\noindent Theorem \ref{theorem1} shows the generalisation error is bounded, thus the monotonic improvement (i.e., Lemma \ref{lemma1}) still holds for FEMRL. There are three key factors affecting the maximal value of generalisation error $\epsilon_m$: the virtual global empirical error $\epsilon_{\hat{S}_k}(h_{\hat{S}_k})$, the number of training samples $m$, and the sum of TVDs between the clients' sample policies and the virtual global sample policy, $\Gamma$. 

Note that, The virtual global empirical error can in principle be estimated and optimised approximately by the training loss. $\Gamma = \sum_{k} D_{\rm TV}(\overline{\pi}_{D} ||\pi_{D}^{k}) = \sum_{k}||D - D_{k}||_1$ can be a measurement of the degree of non-IID of clients' datasets. When the data distribution is IID on all clients, $||D - D_{k}||_1 = 0$, $D_{\rm TV}(\overline{\pi}_{D} ||\pi_{D}^{k})=0$, $\forall k$, which means all clients share the same sample policy. When the data distribution of clients becomes heterogeneous, $\Gamma >0$. Specifically, the higher degree of non-IID of data distribution, the higher $\Gamma$ is.

We now analyse the effect of policy synchronisation rate $\alpha$ on the measure of non-IID client data distributions, $\Gamma$. Denote the sample policy before and after the global update as $\pi_D$ and $\pi'_D$, respectively. After policy synchronisation (with rate $\alpha$),  $\alpha K$ clients have the latest sample policy $\pi'_D$ and $(1-\alpha) K$ clients use the old sample policy $\pi_{D}$. Therefore, the virtual global sample policy is given as:
\begin{equation}
\label{new_avg_pi}
    \overline{\pi}_D = \frac{1}{K} \left [ \sum_{k=1}^{\alpha K} \pi'_{D} + \sum_{k=1}^{(1-\alpha)K}\pi_{D} \right ] = \alpha \pi'_{D} + (1-\alpha) \pi_{D}.
\end{equation}
Using the the definition of $\Gamma$:
\begin{equation}
\label{gamma_non_iid}
\begin{aligned}
 \Gamma &:= \sum_{k=1}^K  D_{\rm TV}(\overline{\pi}_{D} ||\pi_{D}^{k}) \\ 
 &= \sum_{k=1}^{\alpha K} D_{\rm TV}(\overline{\pi}_D||\pi'_D) + \sum_{k=1}^{(1-\alpha) K} D_{\rm TV}(\overline{\pi}_D||\pi_D).
\end{aligned}
\end{equation}
Replacing $\overline{\pi}$ using Eq. (\ref{new_avg_pi}), we have for the synchronised component:
\begin{equation}
\label{first_term_iid}
\begin{aligned}
     D_{\rm TV}(\overline{\pi}_{D} ||\pi'_{D}) &= \frac{1}{2} \sum_{s,a} \left |  \alpha \pi'_{D} + (1-\alpha) \pi_{D} - \pi'_{D} \right | \\
        &= \frac{1}{2} (1-\alpha) D_{\rm TV}(\pi_{D}||\pi'_{D}).
\end{aligned}
\end{equation}
Similarly, for the unsynchronised component:
\begin{equation}
\label{second_term_iid}
\begin{aligned}
     D_{\rm TV}(\overline{\pi}_{D} ||\pi_{D}) &= \frac{1}{2} \sum_{s,a} \left |  \alpha \pi'_{D} + (1-\alpha) \pi_{D} - \pi_{D} \right | \\
        &= \frac{1}{2} \alpha D_{\rm TV}(\pi_{D}||\pi'_{D}).
\end{aligned}
\end{equation}
Combining Eqs. (\ref{gamma_non_iid}), (\ref{first_term_iid}), and (\ref{second_term_iid}), we have
\begin{equation}
    \Gamma = \alpha (1-\alpha) K D_{\rm TV}(\pi_{D}||\pi'_{D}). \label{gamma_noniid_final}
\end{equation}

\noindent Eq. (\ref{gamma_noniid_final}), shows that $\Gamma$ is influenced both by the policy discrepancy $D_{\rm TV}(\pi_{D}||\pi'_{D})$ and the policy synchronous rate $\alpha$. $\Gamma$ takes the maximal value with respect to $\alpha$ at $\alpha = 0.5$. For this value, we would expect the convergence of FEMRL to be most hindered due to highly heterogeneous clients. In the next section, we will show how the degree of non-IID of clients' data distributions affects the rate of the reward improvement for FEMRL. 

\section{Experimental Evaluation}
\label{sec::experiment-sec}
In this section, we evaluate the proposed FEMRL with model-free FRL algorithms in standard RL environments. We first give the implementation details about all the algorithms and environments. Next, we give the comparative assessment about FEMRL. Finally, we investigate the impact of non-IID client data, local update steps, and ensemble knowledge distillation.  

\subsection{Implementation details}
\label{sec::implementation-details}
We evaluate the performance of FEMRL on four realistic continuous control tasks (i.e., HalfCheetah, Ant, Hopper, and Swimmer) from the rllab framework \cite{Mujoco} which are widely used to evaluate the RL algorithms \cite{meng2021ppoaccel,chua2018deep,zhang2019asynchronous}. For all these tasks, we set the maximal episode length to 500. { \color{black}One important application scenario of edge computing is in smart manufacturing where robots are widely used to improve production automation and productivity \cite{chen2018edge}. In the context of smart manufacturing, the proliferation of terminal devices (e.g., mobile robots and mechanical arms) has given rise to new challenges for the real-time operation and maintenance, scalability, and reliability. Edge computing aims to address these challenges by providing edge servers with networking, computing, and storage capabilities close to the manufacturing unit to meet key performance requirements. Therefore, in our experiments, we consider robotics environments for the local learning clients, which together with the simulated edge server can reflect a typical edge computing scenario in smart manufacturing. 
} We assume all environments running on an edge computing platform which includes multiple user devices and an edge server. All user devices share the same environment dynamics as we discussed in section \ref{methodology}. 

We implement FEMRL and all baseline algorithms by using Pytorch ($\ge$1.7.0). Specially, the dyanmics of the MDP is approximated by the feed-forward neural network with two hidden layers and each layer includes 500 units. The activation function at each layer is ReLU. Instead of directly predicting the next state, the network predicts the normalised differences between the next state $s_{t+1}$ and $s_{t}$ as in previous works \cite{kurutach2018model,luo2018algorithmic}. Each client maintains its own normalised statistics (i.e., the mean $\mu$, and standard deviation $\sigma$) based on the sampled local dataset. The normalised difference can be calculated as $((s_{t+1} - s_{t}) - \mu) / \sigma $. The policy neural network is also implemented by a feed-forward neural network with two hidden layers, each of which has 128 hidden units. We use ReLU as the activation function and the output of the policy neural network is a Gaussian distribution $\mathcal{N}(\mu(s), \sigma^2)$ where $\sigma$ is a state-independent trainable vector.

For other default settings of FEMRL, we set the number of inner loops as $n_{\rm inner}=20$ at each training epoch. Each client conducts 500 environment steps using the sample policy and stores the sampled data locally. We set the batch size of local updates for the dynamics model as 128 for all clients. Each client conducts $E=80$ steps of local training with Adam (with learning rate $10^{-3}$) and then uploads parameters of the local dynamics model to the server. The server then aggregates the uploaded models into a single global model through knowledge distillation. Specifically, we use the sample policy to sample trajectories based on the ensemble of client models and then apply the student-teacher scheme to train a single global model on the fictional trajectories. The learning rate and batch size of the knowledge distillation are set as $10^{-3}$ and $128$, respectively. At each epoch, we optimise the dynamics model and policy alternatively for $n_{\rm inner} = 20$ times. At each inner loop,  we conduct $T_{\rm c}=5$ communication rounds between clients and server for training the dynamic models. After the training for the dynamics model, we then use a policy gradient algorithm (TRPO) to train the policy. We set the number of iterations for policy training as $G=20$. 

\begin{figure*}
  \centering
  \includegraphics[width=\textwidth]{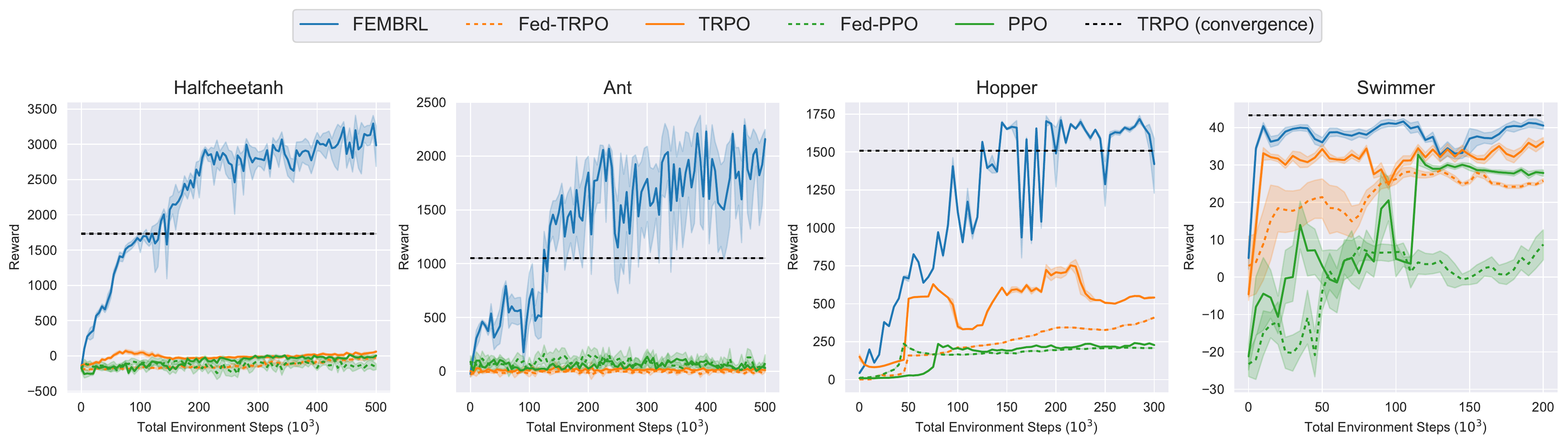}
  \centering \caption{ {\color{black} Global total reward during training for FEMRL (blue) and the four baselines on continuous control benchmarks. Solid curves show the average over 10 trials, and shaded regions show the standard deviation of the mean. The dotted horizontal lines give the final total reward of TRPO after 5 million environment steps.} }
  \label{fig:default_results}
\end{figure*}

{\color{black} In our framework, edge devices run two computation tasks: sampling data from the environment and training the local dynamic models. The sampling process is conducted by forward propagation of the policy network with the linear time complexity $O(n)$. Here $n$ is the size of the sampled data and is usually small for an edge device. In addition, the process of training the local dynamic models is the same as most of federated learning algorithms that have the linear  computation complexity $O(m)$, where $m$ is the number of local training samples. Therefore, the computation overhead of our method is acceptable for edge devices. }

For the settings of the model-free baseline algorithms, we use two advanced policy gradient methods: PPO and TRPO. Both PPO and TRPO use General Advantage Estimator (GAE) \cite{schulman2015high} to measure advantages. The policy networks of all baseline algorithms share the same settings as FEMRL. The hyperparameters settings for centralised TRPO and PPO are listed in Tables \ref{hyperparameters_trpo} and \ref{hyperparameters_ppo}, respectively. Fed-TRPO and Fed-PPO share most of the hyperparameters settings as their centralised counterparts except batch size (TRPO) and environment steps per epoch (PPO). Since Fed-TRPO and Fed-PPO do not collect data from clients, thus the batch size of Fed-TRPO on each client is set as 500 while the environment steps per epoch of Fed-PPO on each client is set as 500. 

\begin{table}[h]
  \renewcommand\arraystretch{1.2}
  \renewcommand{\tabcolsep}{1.5mm}
  \caption{Hyperparameters for TRPO.}
  \label{hyperparameters_trpo}
  \centering
  \begin{tabular}{ c | c || c | c }
    \hline
    \textbf{Hyperparameter} & \textbf{Value} & \textbf{Hyperparameter} & \textbf{Value} \\ 
    \hline
    Batch Size & 5000 & Max KL Divergence & 0.01\\
    \hline
    Discount $\gamma$ & 0.99 & GAE $\lambda$ & 0.95 \\
    \hline
    Conj. Gradient Damping & 0.1 & Conj. Gradient Steps & 10 \\
    \hline
  \end{tabular}
\end{table}

\begin{table}[h]
  \renewcommand\arraystretch{1.2}
  \renewcommand{\tabcolsep}{1.5mm}
  \caption{Hyperparameters for PPO.}
  \label{hyperparameters_ppo}
  \centering
  \begin{tabular}{ c | c || c | c }
    \hline
    \textbf{Hyperparameter} & \textbf{Value} & \textbf{Hyperparameter} & \textbf{Value} \\ 
    \hline
    Batch Size & 100 &  Env. Steps per Epoch & 5000 \\
    \hline
    Learning Rate & 0.001 &  Optimizer & Adam \\
    \hline
    GAE $\lambda$ & 0.95 &  Discount $\gamma$ & 0.99 \\
    \hline
    Ent. Coefficient & 0.01 & Clipping Value $\epsilon$ & 0.2  \\
    \hline
  \end{tabular}
\end{table}

\subsection{Comparative assessment}

We first compare the sampling efficiency of FEMRL to 4 other algorithms: 1) TRPO \cite{schulman2015trust}, a model-free policy-gradient based algorithm running centrally, where all client samples are collected on the server (thus breaking the FL assumption). The policy is updated using the gathered samples. 2) Proximal Policy Optimisation (PPO) \cite{schulman2017proximal}, a model-free RL algorithm also running centrally. 3) Federated TRPO (Fed-TRPO), where each client collects samples from the environment and updates the local policy based on the collected samples. After the local update of the policies on the clients, the server averages all uploaded client policies, creating a global policy for the next round of training. 4) Federated PPO (Fed-PPO), again applying PPO to the FL setting. Both Fed-TRPO and Fed-PPO are model-free FRL methods. The existing federated RL methods, e.g., \cite{FedDeepRL,AWFDRL,liu2019lifelong}, share the same FL architecture as Fed-TRPO and Fed-PPO, but differ in the model-free RL algorithm used.

In FEMRL, after the policy update on the server, the parameters of the policy network are sent to clients to update their local policies. However, the update of local policies at clients can be asynchronous: some clients receive the updated policy, others do not receive it and thus will use the old policy for sampling. As a consequence, clients will have heterogeneous sampling policies. We denote the policy synchronous rate as $\alpha$, where only $\alpha K$ clients will receive the updated sample policy at each training epoch. As the default setting of FEMRL, we set $\alpha = 0.3$, the number of local update steps of FL $E = 80$, and the number of FL communication rounds $T_{\rm c}=5$. We use $K=10$ clients for all algorithms. Each client performs 500 environment steps at each epoch, which therefore has 5000 total environment steps. FEMRL first trains the dynamics model based on the sampled data, and then uses this model to generate fictitious data for policy updating. In contrast, the model-free algorithms (i.e., TRPO, PPO, Fed-TRPO, Fed-PPO) directly use the sampled data for policy update.  Due to the sparse reward signal of RL, they generally require huge numbers of interactions with the environment to obtain effective policies, leading to sample inefficiency.

Fig. \ref{fig:default_results} shows the policy improvement rate of FEMRL and the four baseline algorithms. {\color{black} The learning parameters of all the algorithms in Fig. \ref{fig:default_results} use the default settings as given in the previous paragraph and Section \ref{sec::implementation-details}. } The dotted lines demonstrate the final performance of (centralised) TRPO after 5 million environment steps. The performance of  (centralised) TRPO or PPO acts as a soft upper bound of the federated counterpart (i.e., Fed-TRPO, or Fed-PPO). FEMRL learns substantially faster and achieves the best performance with 0.5 million or fewer environment steps. For example, FEMRL achieves the same performance at 120k environment steps as TRPO does after 5 million environment steps in the HalfCheetah and Ant environments. {\color{black} FEMRL is an FL variant of the model-based RL \cite{luo2018algorithmic}, where we train the environment dynamics model via FL and optimize the policy by interacting with the learned dynamics model directly. Therefore, FEMRL can achieve better sample efficiency than model-free RL methods and their FL variants.}

\subsection{The impact of non-IID client data}

\begin{figure}[t]
    \centering
    \includegraphics[width=3.0in]{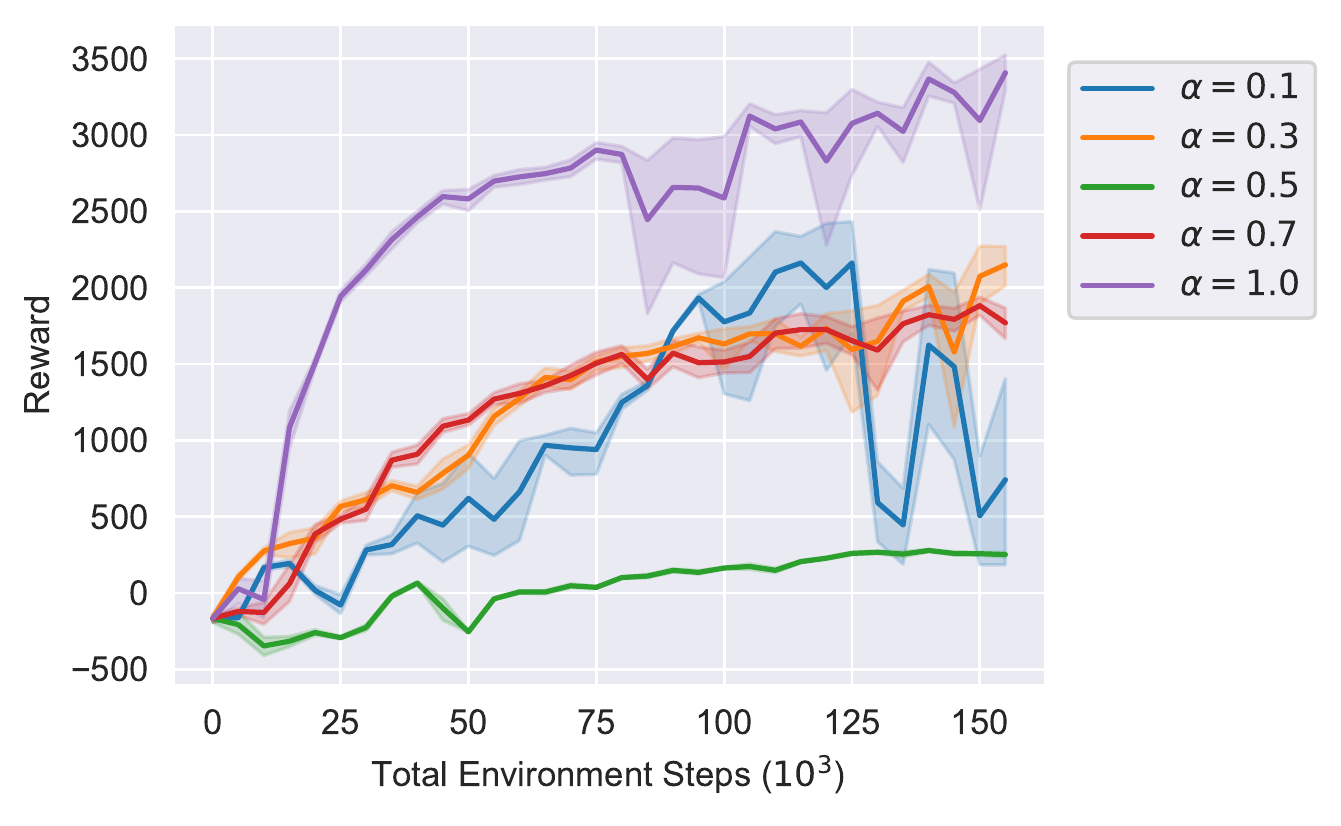}
    \centering\caption{Performance of FEMRL with different policy synchronous rates (i.e., $\alpha$) on HalfCheetah.}
    \label{halfcheetanh-noniid}
\end{figure}

\begin{figure}[t]
  \includegraphics[width=3.0in]{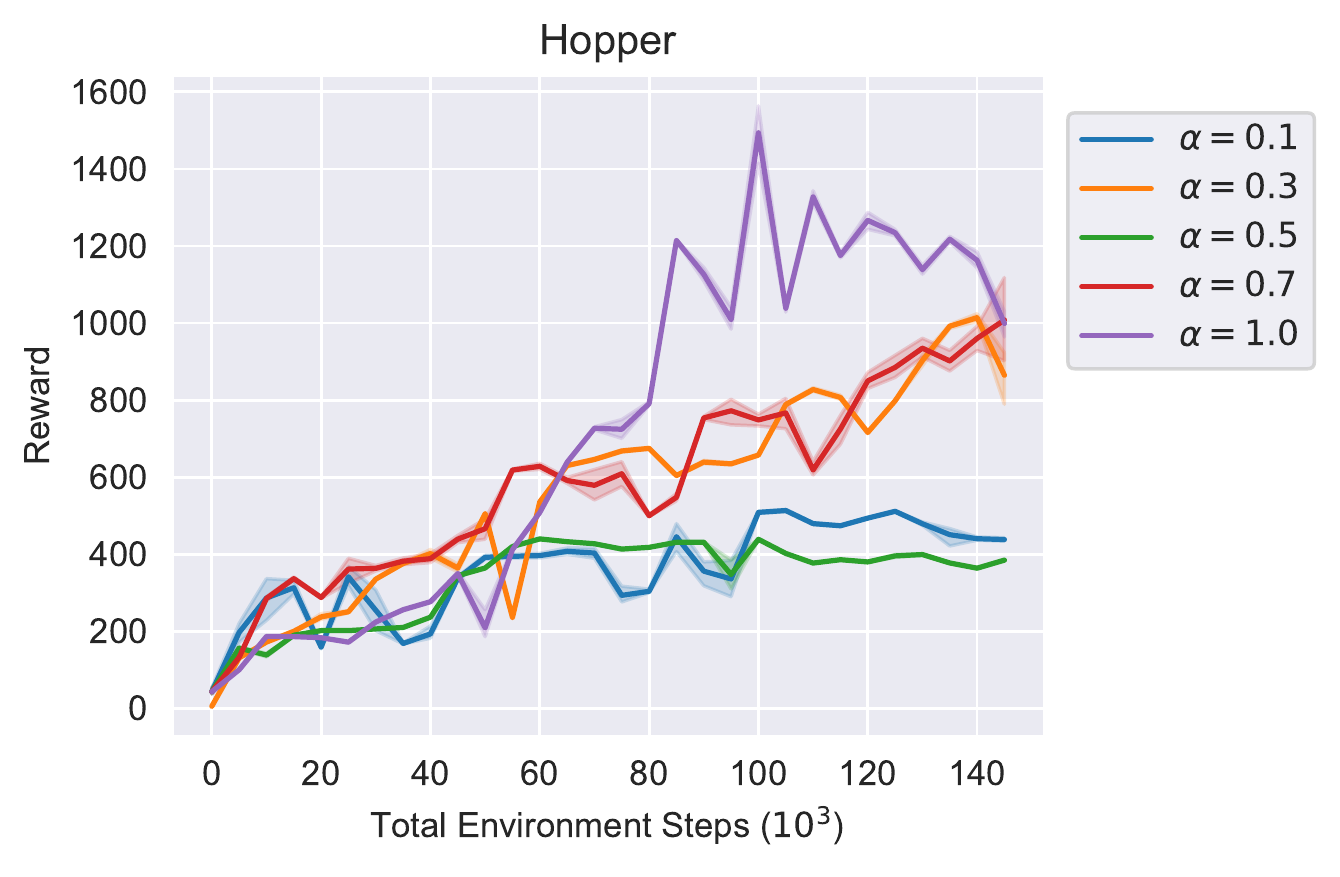}
  \centering\caption{Performance of FEMRL with different policy synchronous rates (i.e., $\alpha$) on Hopper.}
  \label{hopper-ant-noniid}
\end{figure}

Clients with non-IID datasets possess unique, non-identical minimisers to their local objectives. During the local-update phase of FL, each participating client's model will diverge from the global model and move towards their local minimiser. This divergence is termed `client-drift' \cite{SCAFFOLD} and has been extensively shown to harm the performance of the global model. The greater the degree of non-IID client data, and the more local steps clients perform, the greater the level of client-drift. In this section, we investigate how non-IID client data impacts the performance of FEMRL. 

Eq. (\ref{gamma_noniid_final}) shows that the degree of non-IID is determined by $\alpha$ for FEMRL. Therefore, we evaluate FEMRL with varying $\alpha$ on HalfCheetah and Fig. \ref{halfcheetanh-noniid} shows the training curves. When $\alpha=0.5$, the rate of policy improvement is slowest due to the highly non-IID client data: higher model error results in a worse policy. The rate of policy improvement naturally is the fastest when $\alpha = 1.0$, as $\Gamma$ is 0 (according to Eq. (\ref{gamma_noniid_final})) that represents an IID scenario. When $\alpha = 0.1$, although $\Gamma$ is small, performance is still low because the discrepancy (i.e., $\epsilon_\pi$) between the sample policy and target policy is large. Lemma \ref{lemma1} reveals the relationship between $\epsilon_\pi$ and the returns of the dynamics model and the environment. The $\alpha \in \{0.3, 0.7\}$ curves show that the policy improvement rate of FEMRL falls gracefully as $\alpha \to 0.5$. 

Fig. \ref{hopper-ant-noniid} shows the performance of FEMRL on Hopper with varying policy synchronisation rates. As expected, when $\alpha=1.0$, the client data is purely IID, therefore FEMRL can achieve the best performance. In contrast, when $\alpha=0.5$, the degree of non-IID is maximal, therefore, FEMRL obtains the worst performance.

\subsection{The impact of local update steps}
\begin{figure}[t]
    \centering
    \includegraphics[width=3.0in]{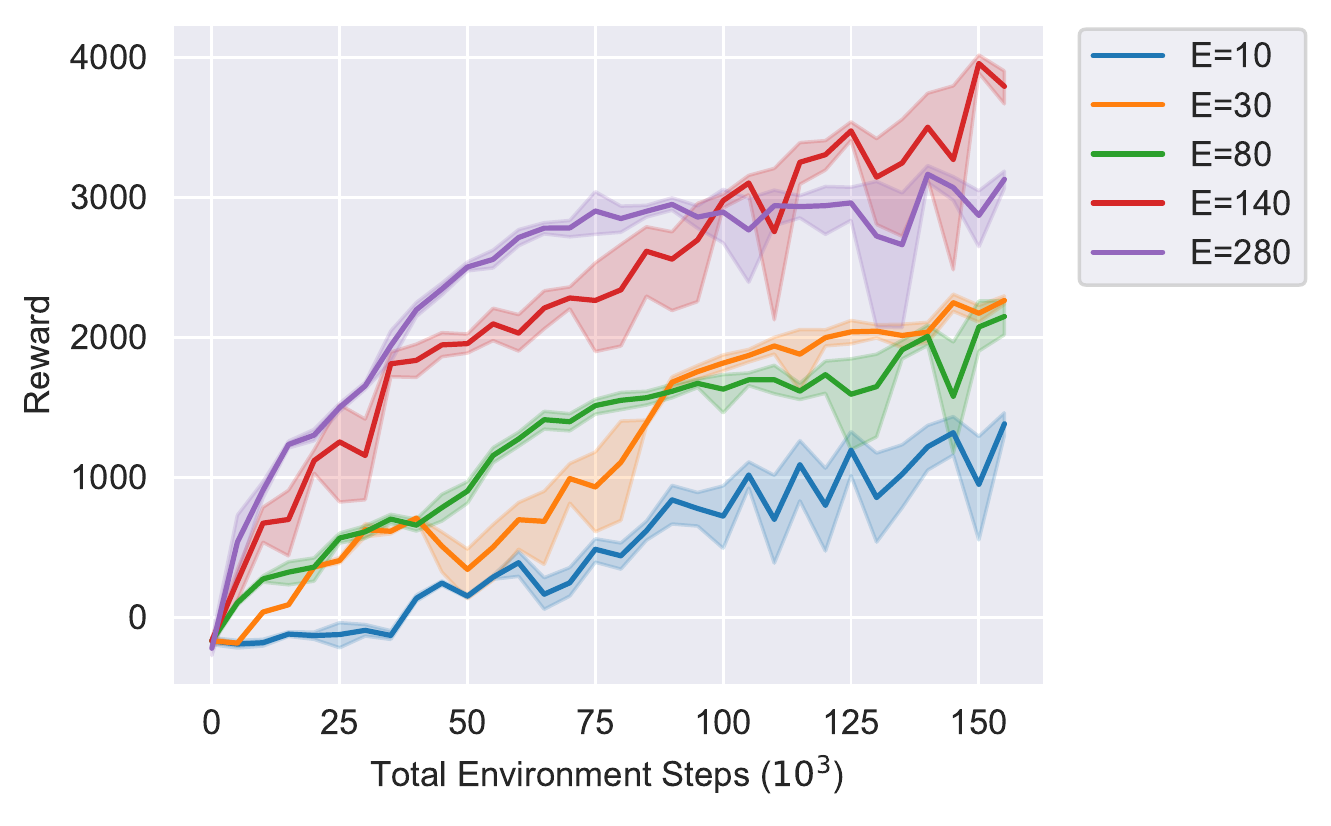}
    \centering\caption{Performance of FEMRL with different numbers of local update steps (i.e., $E$) on HalfCheetah.}
    \label{halfcheetanh-localsteps}
\end{figure}

\begin{figure}[t]
  \includegraphics[width=3.0in]{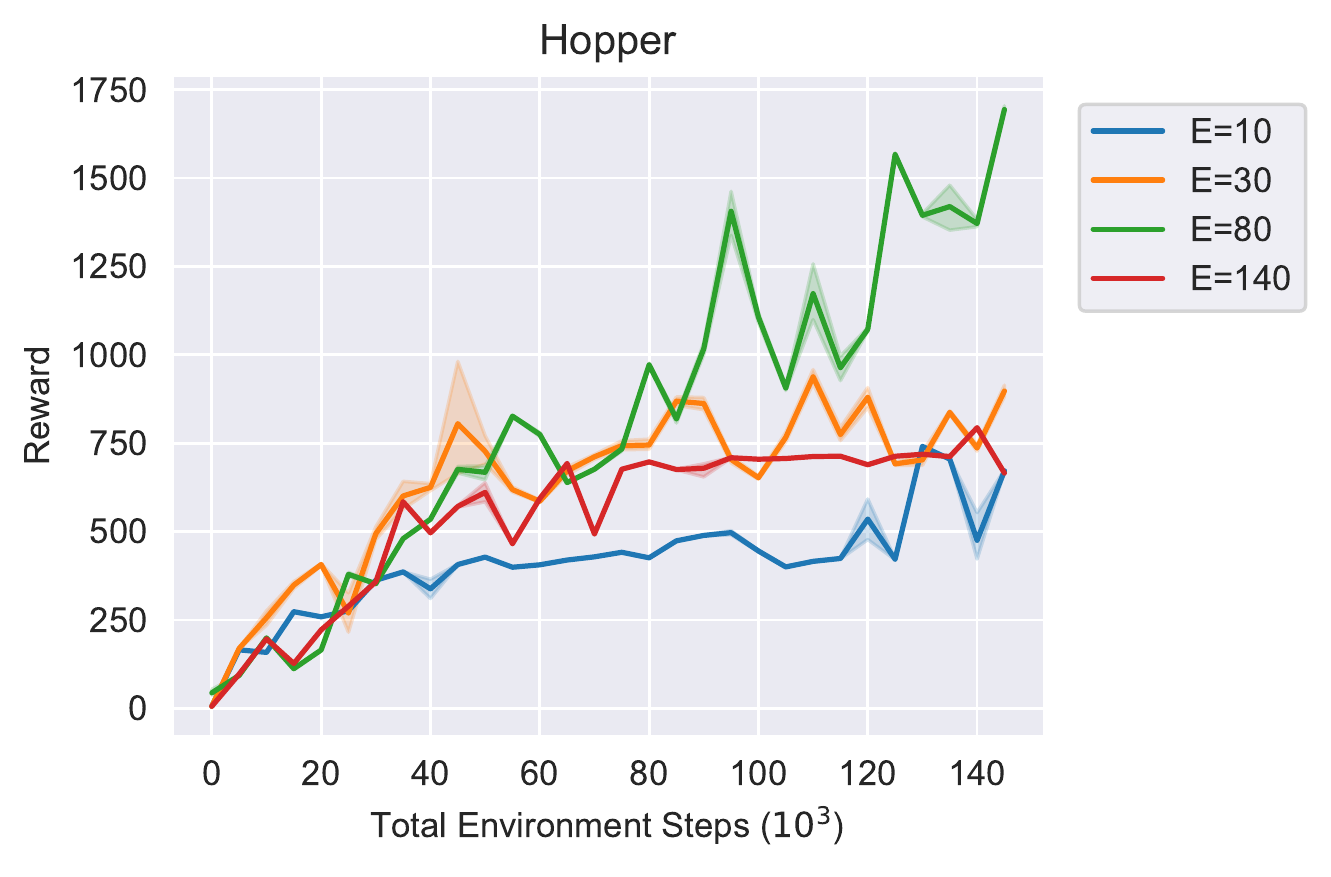}
  \centering\caption{Performance of FEMRL with different numbers of local update steps (i.e., $E$) on Hopper.}
  \label{hopper-ant-local}
\end{figure}

% Finally, we evaluate the performance of FEMRL with different local update steps $E$, which is another key factor that influence the performance of FEMRL in non-IID setting. In generally, large local update steps can speed up the learning process with a fixed communication rounds $T_{\rm c}$. However, too large local update steps can intensify the `client-drift', since the local model tends to overfitting with the wrong minimiser. On the other hand, reduce the number of local update steps can avoid the client drift, but will slow down the learning process. More communication rounds are required for the training, leading to high communication cost. 

Previous works have shown that the number of local steps of SGD that clients perform, $E$, is a key factor affecting the convergence of FL algorithms \cite{fedavg, quagmire, SCAFFOLD}. Larger $E$ allows clients to do more work locally and make more progress, but the final performance of the global model is harmed when the data on clients is non-IID. 

Fig. \ref{halfcheetanh-localsteps} shows the convergence of FEMRL with a varying number of local steps $E$, for a fixed number of communication rounds $T_{\rm c} = 5$. As expected, as $E$ increases, the initial rate of policy improvement increases as clients make more progress in training the dynamics model. However, as $E$ becomes very large ($E = 280$), the final reward plateaus at 3000, as the environment model reaches a local optimum and the mininum error it can achieve is harmed. In this scenario, the value of $E = 140$ strikes a good trade-off between policy improvement rate and maximum reward.

Fig. \ref{hopper-ant-local} shows the performance of FERML on varying numbers of local update steps, for a fixed number of communication rounds $T_c=5$. As expected, both small ($E=10, 30$) and large ($E=140$) number of local update steps can harm the convergence rate. The value of $E=80$ achieves the best performance in this scenario. 

\subsection{The impact of ensemble knowledge distillation}

\begin{figure}[h]
  \centering
  \includegraphics[width=2.5in]{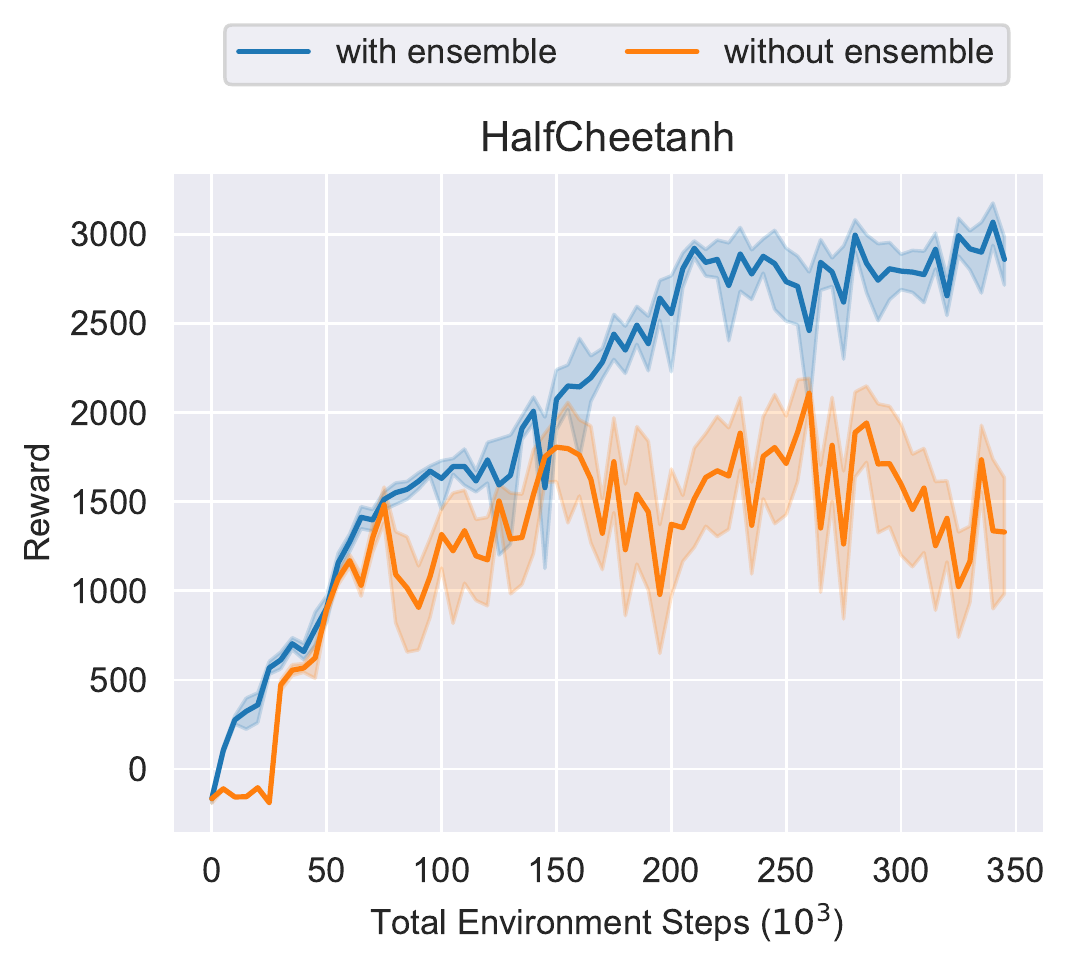}
  \centering\caption{Performance of FEMRL with or without ensemble knowledge distillation on HalfCheetah.}
  \label{halfcheetah-en}
\end{figure}

We then investigate how the ensemble knowledge distillation method affects the performance. We train FEMRL on HalfCheetah without using ensemble knowledge distillation. Specifically, we directly average the uploaded parameters of clients' models and create a single global dynamics model for the FL training process. After $T=5$ rounds of FL training, we use the global dynamics model for the policy-updating process using TRPO. The rest of the hyperparameter settings are  the same as the default settings. Fig. \ref{halfcheetah-en} shows the results of FEMRL on HalfCheetah with or without ensemble knowledge distillation. We find that the ensemble model distillation method can significantly improve the performance of FEMRL, indicating the importance and necessity of combining ensemble knowledge distillation into our method.

\subsection{Discussion} % give a disscusion about how to adapt FEMRL to other scenarios. 
{\color{black}The proposed FEMRL is a general federated RL method that is not limited to a specific problem. We can adapt FEMRL to other edge computing scenarios by modifying the structure of the policy network and dynamics model to fit the dimension of the state and action space of the specific edge computing application. For example, task offloading is a typical edge computing application, which enables to offload computation-intensive tasks of mobile applications from user devices to edge servers. However, unlike the continual action space of robotics control tasks defined in our experiments, the action space of the task offloading problem is generally discrete \cite{mach2017mobile}. To adapt FEMRL to the task offloading problem, we need to redesign the input/output layers of the policy network and dynamics model to fit the discrete state and action space defined in the task offloading problem, especially replacing the output layer of the policy network from a Gaussian distribution to a Categorical distribution. While the training process of FEMRL remains the same. 

Although FEMRL has many benefits to MEC applications, there are several challenges requiring further exploration. In particular, the performance of the trained RL policy might deteriorate when handling fast-changing environments. In fact, how to enhance the generalisation ability of DRL methods for fast-changing environments is still an open problem in RL \cite{dulac2021challenges}. We feel that a useful solution for generalisation objectives would constitute a whole new paper, so we leave this to future work: we intend to combine meta-learning \cite{jeong2020ood,nagabandi2018learning} into our framework to solve the out-of-distribution problem for enhancing its generalisation ability. 
}

\section{Conclusion}
\label{conclusion}
In this paper, we proposed a novel federated RL algorithm, FEMRL, for edge computing systems, which incorporates model-based RL, and ensemble distillation technologies into FL. In FEMRL, clients train their local dynamics model based on their locally sampled data. An ensemble of the dynamics models is then created at the edge server based on the updated local models. We use the ensemble model for both policy training and FL model aggregation (by an ensemble distillation method). The updated policy is then sent to clients for the next-round of sampling process. We provide a rigorous theoretical analysis to prove the monotonic improvement of FEMRL in federated setting with non-IID client data. Finally, we evaluate FEMRL based on four challenging continuous control tasks. Experiment results demonstrate that FEMRL can achieve much higher sample efficiency than federated model-free counterparts.

\section{Acknowledgement}
\label{sec::acknowledgement}
This work was supported in part by EU Horizon 2020 INITIATE Project under Grant 101008297, in part by Royal Society International Exchanges Project under Grant IEC/NSFC/ 211460, in part by EPSRC New Horizons fund EP/X019160/1, and in part by UKRI Project EP/X038866/1.

\bibliographystyle{IEEEtran} 
\bibliography{fedmbrl-ieee}
%\vspace{-10 mm}

\clearpage

\setcounter{secnumdepth}{2}
\appendices
\renewcommand{\thesection}{\Alph{section}}
\section{Monotonic Improvement Guarantee}
In this section, we first provide some useful Lemmas for the theoretical analysis of monotonic improvement guarantee for FEMRL and then give the proof of Lemma \ref{lemma1}.
\subsection{Lemmas}
\begin{lemma}
\label{lemma:isinequality}
(Importance sampling inequality) For any distribution $\rho(s)$ and $\rho'(s)$ and a function $f(s)$, we have $\mathbb{E}_{s \sim \rho(s)}f(s) \leq \mathbb{E}_{s \sim \rho'(s)}f(s) + |\rho(s) - \rho'(s)|f_{\rm max}$, where $f_{\rm max}$ is the maximal value of $f(s)$.
\end{lemma}

\begin{proof}
\begin{equation*}
\begin{aligned}
    \mathbb{E}_{s \sim \rho(s)}f(s) &=  \mathbb{E}_{s \sim \rho'(s)}\frac{\rho(s)}{\rho'(s)}f(s) \\
    &= \mathbb{E}_{s \sim \rho'(s)}\frac{\rho(s)-\rho'(s) + \rho'(s)}{\rho'(s)}f(s) \\
    &=\mathbb{E}_{s \sim \rho'(s)}f(s) +\mathbb{E}_{s \sim \rho'(s)}(\rho(s)-\rho'(s))f(s) \\
    &\leq  \mathbb{E}_{s \sim \rho'(s)}f(s) + \sum_{s} |\rho(s)-\rho'(s)| f_{\rm max} \\
    &\leq \mathbb{E}_{s \sim \rho'(s)}f(s) + ||\rho(s)-\rho'(s)||_1 f_{\rm max}.
\end{aligned}
\end{equation*}
\end{proof}

\begin{lemma}
\label{lemma:boundeddiscount}
(Bounded difference of discounted state distributions). Let $\pi$ and $\pi_{D}$ be two different policies and $ \epsilon_{\pi} = D_{\rm TV}(\pi || \pi_{D})$, we have:

\begin{equation*}
    \| \rho_{\pi}^{\mathcal{M}} - \rho_{\pi_{D}}^{\mathcal{M}} \|_1 \leq \frac{2\gamma}{(1-\gamma)^2} \epsilon_{\pi}.
\end{equation*}

\end{lemma}

\begin{proof}
Define $\mathbb{P}_{\pi}^{\mathcal{M}}$ and $\mathbb{P}_{\pi_{D}}^{\mathcal{M}}$ as the transition kernels of the MDP $\mathcal{M}$ following policies $\pi$ and $\pi_{D}$, respectively. Let $\mathbf{G} = ( \mathbf{I} + \gamma \mathbb{P}_{\pi}^{\mathcal{M}} + (\gamma \mathbb{P}_{\pi}^{\mathcal{M}})^2 + ...  ) = (\mathbf{I} - \gamma \mathbb{P}_{\pi}^{\mathcal{M}})^{-1}$ and $\mathbf{G_D} = ( \mathbf{I} + \gamma \mathbb{P}_{\pi_{D}}^{\mathcal{M}} + (\gamma \mathbb{P}_{\pi_{D}}^{\mathcal{M}})^2 + ...  ) = (\mathbf{I} - \gamma \mathbb{P}_{\pi_{D}}^{\mathcal{M}})^{-1}$. Let $\mathbf{\Delta} = \mathbb{P}_{\pi_{D}}^{\mathcal{M}} - \mathbb{P}_{\pi}^{\mathcal{M}}$. We start with some algebraic manipulations as:

$$\mathbf{G}^{-1} - \mathbf{G_D}^{-1} = (\mathbf{I} - \gamma \mathbb{P}_{\pi}^{\mathcal{M}}) - (\mathbf{I} - \gamma \mathbb{P}_{\pi_{D}}^{\mathcal{M}}) = \gamma \mathbf{\Delta}.$$

\noindent Left-multiplying by $\mathbf{G}$ and right-multiplying by $\mathbf{G_D}$, then multiplying by $\rho_0$:
\begin{equation*} 
\mathbf{G_D}\rho_0 - \mathbf{G}\rho_0 =  \gamma \mathbf{G} \mathbf{\Delta} \mathbf{G_D} \rho_0 .
\end{equation*}

\noindent Note that $\rho_{\pi}^{\mathcal{M}} = \mathbf{G}\rho_0$. By definition we have $\|\mathbf{G}\|_1 = (1 - \gamma)^{-1}$, $\| \mathbf{\Delta}\|_1 = 2 D_{\rm TV} (\pi || \pi_{D}) $, and $\| \rho_0\| = 1$. Hence:
\begin{equation*}
\begin{aligned}
    \| \rho_{\pi}^{\mathcal{M}} - \rho_{\pi_{D}}^{\mathcal{M}} \|_1 &= \| \gamma \mathbf{G} \mathbf{\Delta} \mathbf{G_D} \rho_0 \|_1 \leq \gamma \|\mathbf{G}\|_1 \| \mathbf{\Delta}\|_1 \|\mathbf{G_D}\|_1 \|\rho_0\|_1 \\
    &\leq \frac{2 \gamma}{(1-\gamma)^2}  D_{\rm TV} (\pi || \pi_{D}) = \frac{2 \gamma }{(1-\gamma)^2 } \epsilon_{\pi}.
\end{aligned}
\end{equation*}

\end{proof}

\subsection{Proof of Lemma \ref{lemma1}}
\label{appendix:lemma1s}

\begin{proof}[Proof] 
Let $\rho_{\pi}^{\mathcal{M}}$ be the discounted visitation frequencies \cite{schulman2015trust} over the state space as  $\rho_{\pi}^{\mathcal{M}}(s) = \sum_{t=0}^{\infty} \gamma^{t}P(S_t=s|\pi,\mathcal{M})$, where $P(S_t=s|\pi,\mathcal{M})$ denotes the probability of being in state $s$ at time step $t$ in the MDP $\mathcal{M} := (\mathcal{S}, \mathcal{A}, T, R, \rho_0, \gamma)$ following the policy $\pi$. We can define the expected discounted return as: 
\begin{equation}
\begin{aligned}
V_{\pi}^{\mathcal{M}}   &= \mathop{\mathbb{E}}\limits_{ \tiny \begin{array}{c} S_{t+1} \sim T(\cdot | S_t, A_t) \\ A_t \sim \pi(\cdot | S_t) \end{array} } \left [\sum_{t=0}^{\infty} \gamma^t R(S_{t}, A_{t}) \bigg\rvert S_0 = s_0 \right ] \\
                        &= \mathbb{E}_{s \sim \rho_{\pi}^{\mathcal{M}}(s), a \sim \pi(a|s)} \left [ R(s,a) \right ],
\end{aligned}
\label{discount_ret}
\end{equation}
where $s_0$ is the initial state.
\\

\noindent Let $W_j$ be the discounted total reward when executing $\pi$ on the dynamics model $\mathcal{M}$ for the first $j$ steps and the rest of the steps on $\widehat{\mathcal{M}}$. That is:
\begin{equation*}
    W_{j} = \mathop{\mathbb{E}}\limits_{\tiny \begin{array}{c} \forall t \geq 0, A_t \sim \pi(\cdot|S_t) \\ \forall j > t \geq 0, S_{t+1}\sim T(\cdot|S_t, A_t) \\ \forall t \geq j, S_{t+1} \sim \widehat{T}(\cdot|S_t, A_t) \end{array}} \left[ \sum_{t=0}^{\infty} \gamma^t R(S_t, A_t) | S_0 = s_0 \right] .
\end{equation*}

\noindent Note that we define $V_{\pi}^{\mathcal{M}} = \mathbb{E}_{s_0 \sim \rho_0} \left [ V_{\pi}^{\mathcal{M}}(s_0) \right ]$. By definition we have $W_{\infty} = V_{\pi}^{\mathcal{M}}$ and $W_0 = V_{\pi}^{\widehat{\mathcal{M}}}$, thus:

    \begin{equation*}
        V_{\pi}^{\widehat{\mathcal{M}}} - V_{\pi}^{\mathcal{M}} = \sum_{j=0}^{\infty} \left ( W_{j+1} - W_j \right ).
    \end{equation*}

\noindent We rewrite $W_j$ and $W_{j+1}$ as:
\begin{equation*}
    W_j = R_j + \mathop{\mathbb{E}}\limits_{A_j, S_j \sim \pi, T} \left [ \mathop{\mathbb{E}}\limits_{\hat{S}_{j+1} \sim \widehat{T}(\cdot|S_j, A_j) } \left [ \gamma^{j+1} V_{\pi}^{\widehat{\mathcal{M}}}(\hat{S}_{j+1}) \right ] \right ],
\end{equation*}

\begin{equation*}
    W_{j+1} = R_j + \mathop{\mathbb{E}}\limits_{A_j, S_j \sim \pi, T} \left [ \mathop{\mathbb{E}}\limits_{S_{j+1} \sim T(\cdot|S_j, A_j) } \left [ \gamma^{j+1} V_{\pi}^{\mathcal{M}}(S_{j+1}) \right ] \right ].
\end{equation*}

\noindent we define $G^{\pi}_{\widehat{\mathcal{M}}} (S,A) =  \mathop{\mathbb{E}}\limits_{S' \sim T(\cdot|S, A) } \left [ V_{\pi}^{\mathcal{M}}(S') \right ] - \mathop{\mathbb{E}}\limits_{\hat{S}' \sim \widehat{T}(\cdot|S, A) } \left [  V_{\pi}^{\widehat{\mathcal{M}}}(\hat{S}') \right ] $. Therefore:
\begin{equation*}
\begin{aligned}
    W_{j+1} - W_j &= \gamma^{j+1} \mathop{\mathbb{E}}\limits_{A_j, S_j \sim \pi, T} \left [  G^{\pi}_{\widehat{\mathcal{M}}} (S, A) \right ],
\end{aligned}
\end{equation*}
where $R_j$ is the expected accumulative reward of the first $j$ steps, which are taken w.r.t. dynamics model $\mathcal{M}$. Thus we have:
\begin{equation*}
\begin{aligned}
        V_{\pi}^{\widehat{\mathcal{M}}} - V_{\pi}^{\mathcal{M}} &= \sum_{j=0}^{\infty} \left ( W_{j+1} - W_j \right ) \\
         &= \sum_{j=0}^{\infty} \gamma^{j+1} \mathop{\mathbb{E}}\limits_{A_j, S_j \sim \pi, T}  \left [ G^{\pi}_{\widehat{\mathcal{M}}} (S, A) \right ]  \\
         &= \gamma \mathop{\mathbb{E}}\limits_{\tiny \begin{array}{c} S \sim \rho_{\pi}^{\mathcal{M}}, \\ A \sim \pi(\cdot|S) \end{array}}\left [ G^{\pi}_{\widehat{\mathcal{M}}} (S, A) \right ],
\end{aligned}
\end{equation*}
where the last equality is from applying Eq. \eqref{discount_ret}. For simplicity, we define $T(S,A) = T(s'|s,a)$ as the dynamics of the environment and $\widehat{T}(S,A) = \widehat{T}(s'|s, a)$ as the dynamics of the learned model. The reward function is bounded by $r_{\rm max}$ according to Assumption 3, we then have for any value function: $||V_{\pi}|| \leq \frac{1}{1-\gamma} r_{\rm max}$. Next, we bound $G^{\pi}_{\widehat{\mathcal{M}}} (S,A)$ as:
    \begin{equation*}
    \begin{aligned}
        G^{\pi}_{\widehat{\mathcal{M}}} (S,A) &= \sum_{S'} T(S,A) V_{\pi}^{\mathcal{M}}(S') - \sum_{S'}\widehat{T}(S,A) V_{\pi}^{\widehat{\mathcal{M}}} (S') \\ 
        &\leq \frac{r_{\rm max}}{1-\gamma} \sum_{S'} \left [ T(S,A) - \widehat{T}(S,A) \right ] \\
        &\leq \frac{2 r_{\rm max}}{1 - \gamma} D_{\rm TV}(T(S,A) \| \widehat{T}(S,A)).
    \end{aligned}
    \end{equation*}
    
\noindent Therefore:
\begin{equation}
\label{appendix:value_diff}
\begin{aligned}
    V_{\pi}^{\widehat{\mathcal{M}}} - V_{\pi}^{\mathcal{M}} 
     &\leq \frac{2 \gamma r_{\rm max}}{1-\gamma} \mathop{\mathbb{E}}\limits_{\tiny \begin{array}{c} S \sim \rho_{\pi}^{\mathcal{M}}, \\ A \sim \pi(\cdot|S) \end{array}}\left [ D_{\rm TV}(T(S,A) \| \widehat{T}(S,A)) \right ].
\end{aligned}
\end{equation}
    
\noindent We define $\epsilon_m = \mathbb{E}_{S \sim \rho_{\pi_{D}}^{\mathcal{M}}, A \sim \pi(\cdot|S)} \left [ D_{\rm TV} \left ( T(S,A) \| \widehat{T}(S,A) \right ) \right ]$ and $\epsilon_{m}^{\rm max} = \max_{S \sim \rho_{\pi_{D}}^{\mathcal{M}}}\left [ D_{\rm TV} \left ( T(S,A) \| \widehat{T}(S,A) \right ) \right ]$. In our algorithm we use the sample policy $\pi_{D}$ to sample trajectories from the environment, so we bound the following using \textbf{Lemma \ref{lemma:isinequality}} and \textbf{Lemma \ref{lemma:boundeddiscount}} as: 
    \begin{equation*}
        \begin{aligned}
        \mathop{\mathbb{E}}\limits_{\tiny \begin{array}{c} S \sim \rho_{\pi}^{\mathcal{M}}, \\ A \sim \pi(\cdot|S) \end{array}} &\left [ D_{\rm TV} \left (T(S,A) \| \widehat{T}(S,A) \right ) \right ] \\
        & \leq \mathop{\mathbb{E}}\limits_{\tiny \begin{array}{c} S \sim \rho_{\pi_{D}}^{\mathcal{M}}, \\ A \sim \pi(\cdot|S) \end{array}}\left [ D_{\rm TV} \left ( T(S,A) \| \widehat{T}(S,A) \right ) \right ]  \\ 
        &+ \left \| \rho_{\pi}^{\mathcal{M}} - \rho_{\pi_{D}}^{\mathcal{M}} \right \|_{1} \max_{S \sim \rho_{\pi_{D}}^{\mathcal{M}}} \left [ D_{\rm TV}\left (T(S,A) \| \widehat{T}(S,A) \right ) \right ] \\
        & \leq \epsilon_m + \frac{2 \gamma }{(1-\gamma)^2} \epsilon_{\pi} \epsilon_{m}^{\rm max}.
        \end{aligned}
    \end{equation*}

\noindent Combining the above inequality with Eq. (\ref{appendix:value_diff}), we have:

\begin{equation*}
    V_{\pi}^{\widehat{\mathcal{M}}} - V_{\pi}^{\mathcal{M}} \leq \frac{2 \gamma r_{\rm max }}{1-\gamma} \epsilon_m + \frac{4 \gamma^2 r_{\rm max }}{(1-\gamma)^3} \epsilon_{\pi} \epsilon_{m}^{\rm max}.
\end{equation*}
    
\end{proof}

\section{Generalisation analysis of the ensemble dynamics}
\label{appendix:gen}
In this section, we derive a bound on the generalisation error of the environment model trained during our FEMRL learning process. Since the training of the model is a supervised learning process, we can utilise the Probably Approximately Correct (PAC) learning framework for our analysis. First, we give the general bounds for Vapnik–Chervonenkis (VC)-dimension and the discrepancy of the generalisation error between two different data domains. We then give the proof of Theorem \ref{theorem1}.

\subsection{Preliminaries}
\begin{theorem}
(Uniform VC-dimension error bound \cite{mlfoundations}) Let $\mathcal{H}$ be a hypothesis class with $VCdim(\mathcal{H}) \leq d < \infty$. Let $D$ be the probability measures over the sample space. Let $S$ be the empirical dataset sampled from $D$, $S \sim D^m$ where $m$ is the size of the dataset. Then for any $\delta > 0$, with probability at least $1 - \delta$, the following holds for all $h \in \mathcal{H}$:
\begin{equation}
\label{PAC_general}
|\epsilon_{D}(h) - \epsilon_{S}(h)| \leq C\sqrt{\frac{d+\log(1/\delta)}{m}},
\end{equation}
where $C$ is a constant factor. 
\end{theorem}

\noindent We now give a bound of learning between different domains. 

\begin{lemma}
Let $\mathcal{H}$ be a hypothesis class. $D$ and $D'$ denote two probability measures over the sample space. Let $\epsilon_{D}{h}$ denote the general error of $h$ over D. If the loss function $l(\cdot)$ is bounded by $L$, then for every $h \in \mathcal{H}$ we have:
\begin{equation}
\label{domain_adapt_func}
\epsilon_{D}(h) \leq \epsilon_{D'}(h) + L || D-D' ||_1.
\end{equation}

\end{lemma}
\begin{proof}
\begin{equation}
\begin{aligned}
\epsilon_{D}(h) &\leq \epsilon_{D'}(h) + |\epsilon_{D}(h) - \epsilon_{D'}(h)| \\
                &\leq \epsilon_{D'}(h) + \int{ \left | l(y, h(x)) \right | |\mathbb{P}_{(x,y) \sim D} - \mathbb{P}_{(x,y) \sim D'}|} \\
                & = \epsilon_{D'}(h) + L ||D-D'||_1.
\end{aligned}
\end{equation}
\end{proof}

\subsection{Proof of Theorem \ref{theorem1}}

\begin{proof}
According to the definition of Empirical Risk Minimisation (ERM), we have $\epsilon_{S_k}(h_{S_k}) \leq \epsilon_{S_k}(h_{\hat{S}})$, where $h_{\hat{S}}$ is the model learned based on the virtual global empirical dataset $\hat{S}$, where $\hat{S} = \frac{1}{K}\sum_{k=1}^{K} S_{k}$. Therefore, we have:

\begin{equation}
    \frac{1}{K}\sum_{k=1}^{K} \epsilon_{S_{k}}(h_{S_k}) \leq \frac{1}{K}\sum_{k=1}^{K} \epsilon_{S_{k}}(h_{S}) = \epsilon_{\hat{S}}(h_{\hat{S}}).
\end{equation}

\noindent Next we give the bound of the generalisation error of the model ensemble, by considering the distance between the generalisation error of the ensemble of client models, $\epsilon_{D}(\frac{1}{K} \sum_{k}{h_{S_k}} )$, and the generalisation error of the model learned from the virtual global dataset, $\epsilon_{D}(h_{S_k})$. By convexity of the loss function $f$ and Jensen's inequality, we have the probability of at least $1-\delta$ over $\{ S_k \sim D_{k}^m\}_{k=1}^K$ that:

%We have probability of at least $1-\delta$ over $\{ S_k \sim D_{k}^m\}_{k=1}^K$ that

\begin{equation*}
\begin{aligned}
    \epsilon_{D} & \left ( \frac{1}{K} \sum_{k}{h_{S_k}}  \right )  \leq \frac{1}{K} \sum_{k} \epsilon_{D}(h_{S_k})  \\
%                                                  &\leq \frac{1}{K} \sum_{k} \left ( \epsilon_{D_k}(h_{S_k}) + L ||D - D_k||_1 \right ) \\
                                                  &\leq \frac{1}{K} \sum_{k} \left ( \epsilon_{S_k}(h_{S_k}) + C \sqrt{\frac{d + log(1/\delta)}{m}} + L ||D - D_k||_1  \right ) \\
                                                  &\leq \frac{1}{K} \sum_{k}{\epsilon_{S_k}(h_{S_k})} + C \sqrt{\frac{d + log(1/\delta)}{m}} + \frac{1}{K} \sum_{k} L ||D - D_k||_1 \\
                                                  &\leq \epsilon_{\hat{S}_k}(h_{\hat{S}_k}) + C \sqrt{\frac{d + log(1/\delta)}{m}} + \frac{L}{K} \sum_{k}  ||D - D_k||_1.
\end{aligned}
\end{equation*}
The distribution of the virtual global dataset can be calculated using $D = \mathbb{P}_{s,a,s'} = \sum_{s,a} T(s'|s,a) \overline{\pi}_{D}(a|s)$:

\begin{equation}
\begin{aligned}
    ||D - D_{k}||_1 &= \sum_{s',s,a} \left | \mathbb{P}_{s,a,s' \sim D} - \mathbb{P}_{s,a,s' \sim D_{k}} \right | \\
                    &= \sum_{s'}\sum_{s,a} (T(s'|s,a)\overline{\pi}_{D}(a|s) - T(s'|s,a)\pi_{D}^{k}(a|s)) \\
                    &= \sum_{s'}{T(s'|s,a)\sum_{s,a}(\overline{\pi}_{D}(a|s) - \pi_{D}^{k}(a|s))} \\
                    &= \sum_{s,a}(\overline{\pi}_{D}(a|s) - \pi_{D}^{k}(a|s)) \\
                    &= D_{\rm TV}(\overline{\pi}_{D} ||\pi_{D}^{k}).
\end{aligned}
\end{equation}

\noindent Denote the discrepancy between the sample policy of client $k$, $\pi_{D}^{k}$, and the virtual global sample policy $\overline{\pi}_{D}$ as $D_{\rm TV}(\overline{\pi}_{D} ||\pi_{D}^{k})$. Let $\Gamma = \sum_{k} D_{\rm TV}(\overline{\pi}_{D} ||\pi_{D}^{k})$. Therefore, we have:
\begin{equation}
    \epsilon_{D} \left ( \frac{1}{K} \sum_{k}{h_{S_k}} \right )  \leq \epsilon_{\hat{S}_k}(h_{\hat{S}_k}) + C \sqrt{\frac{d + log(1/\delta)}{m}} + \frac{L}{K} \Gamma,
\end{equation}
where $\Gamma$ can be used to measure the degree of the non-IID client data.

\end{proof}

\end{document}